\documentclass{article}

\PassOptionsToPackage{authoryear,round}{natbib}
 \usepackage[preprint]{neurips_2026}

\usepackage[utf8]{inputenc} 
\usepackage[T1]{fontenc}    
\usepackage{hyperref}       
\usepackage{url}            
\usepackage{booktabs}       
\usepackage{amssymb}       
\usepackage{nicefrac}       
\usepackage{microtype}      
\usepackage{xcolor}         
\usepackage{graphicx}
\usepackage{enumitem}
\usepackage{amsmath}
\usepackage{amsthm}
\usepackage{makecell}
\usepackage{float}

\newtheorem{theorem}{Theorem}[section]
\newtheorem{lemma}[theorem]{Lemma}
\newtheorem{proposition}[theorem]{Proposition}
\newtheorem{corollary}[theorem]{Corollary}

\theoremstyle{definition}
\newtheorem{definition}{Definition}[section]

\theoremstyle{remark}

\usepackage{pifont}
\title{EnerGNN: Learning Optimization-Compatible Energy Functions for Exact Constrained Combinatorial Inference}

\author{%
Niki Triantafyllou\thanks{Corresponding author.}
\qquad
Maria M.~Papathanasiou \\
Sargent Centre for Process Systems Engineering, Department of Chemical Engineering \\
Imperial College London \\
\texttt{\{niki.triantafyllou20,maria.papathanasiou11\}@imperial.ac.uk}
}

\begin{document}

\maketitle


\begin{abstract}
Learning-based approaches to combinatorial optimization often predict
solutions directly or use predictions to restrict search, while
expressive energy-based models typically rely on approximate inference.
We introduce EnerGNN, a framework that learns energy functions over
complete binary configurations and minimizes them under explicit hard
constraints through mixed integer programming (MIP).
EnerGNN combines a graph neural network encoder with an
optimization-compatible readout and a nonlinear energy head,
enabling exact constrained inference over the learned energy.
An optimality-aware contrastive loss trains the energy to reflect
solution quality and constraint violations.
Across Quadratic Knapsack, Combinatorial Auctions, and a personalized
cell therapy supply chain problem, EnerGNN achieves near-optimal
solutions with 100\% feasibility.
On the supply chain problem, it achieves a 0.05\% mean primal gap
and a 17.4$\times$ speedup over the monolithic MIP.
The learned energy also generalizes to larger instances and supports
zero-shot constraint injection, adapting to changing constraints
without retraining while preserving feasibility.
\end{abstract}

\section{Introduction}

A central challenge in energy-based structured prediction is to combine expressive learned representations with tractable inference over structured outputs. Energy-based models (EBMs) assign scalar energies to complete output configurations, allowing dependencies and interactions between output variables to be captured through learned energy functions \citep{lecun2006tutorial, belanger2016structured}. However, minimizing expressive neural energies over discrete decision spaces is generally challenging, particularly when solutions must satisfy hard operational, physical, or logical constraints. Mixed integer programming (MIP) provides a framework for enforcing such constraints and obtaining globally optimal solutions, but learned energy functions must admit suitable mathematical representations to be optimized directly. This creates an interface problem between neural learning and exact optimization: \emph{How can a neural model learn expressive energy functions that an exact optimizer can minimize under hard constraints?}

Existing approaches address different aspects of this challenge. In learning-based combinatorial optimization, neural models have been used to construct solutions, predict variable assignments, and guide exact solvers \citep{bengio2021machine,khalil2017learning,gasse2019exact}. Methods such as Neural Diving \citep{nair2020solving} and ConPaS \citep{huang2024contrastive} use learned variable-level predictions to identify promising regions of the solution space, which are then explored through restricted optimization. In energy-based structured prediction, inference with expressive neural energies commonly relies on continuous relaxations, gradient-based optimization, or approximate decoding \citep{belanger2016structured, graber2019graph}. Meanwhile, ReLU neural networks can be represented exactly using mixed integer formulations \citep{fischetti2018deep, anderson2020strong}, providing a mechanism for optimizing learned nonlinear functions under explicit constraints. Together, these developments motivate learning energy functions that are designed to be directly optimized under explicit hard constraints, bringing expressive structured prediction and exact combinatorial inference into a unified framework.

We introduce \textbf{EnerGNN}, an energy-based framework that learns optimization-compatible energy functions over complete binary configurations and minimizes them subject to explicit hard constraints. As illustrated in Figure~\ref{fig:architecture}, EnerGNN combines a bipartite message-passing graph neural network (GNN) encoder with a decision-dependent readout and a piecewise linear ReLU energy head. The GNN encoder produces representations of the problem instance that are computed once and remain fixed during optimization. The readout combines these fixed representations with the binary decision variables to construct a configuration-dependent representation. Together, the readout and nonlinear ReLU energy head allow EnerGNN to model interactions between decision variables while admitting an exact MIP representation. At inference, the learned energy is minimized over the full encoded feasible region using a MIP solver, yielding a feasible solution that is globally optimal with respect to the learned energy when solved to optimality.

A key feature of EnerGNN is the separation of learned preferences from explicit feasibility constraints. The energy-based model learns to distinguish high-quality configurations from suboptimal alternatives through optimality-aware contrastive training, while hard constraints are enforced directly during inference. This separation allows the same learned energy to be optimized under different feasible regions, without retraining the model. In particular, deployment-time constraints can be added, removed, or modified by changing the inference MIP, enabling zero-shot constraint injection. This approach is complementary to neural algorithmic reasoning \citep{velivckovic2021neural}: rather than learning to execute an optimization algorithm, EnerGNN learns an energy function over structured decisions and delegates constrained inference to an exact optimizer.

Our contributions are threefold. First, we formulate constrained combinatorial inference as learning an optimization-compatible energy over complete binary configurations, separating learned preferences from explicit hard constraints. Second, we develop a graph-based energy architecture that captures nonlinear interactions between decision variables while admitting exact MIP representation and globally optimal inference with respect to the learned energy. Third, we develop an optimality-aware contrastive learning objective using feasible and infeasible candidate configurations to learn solution-level energy landscapes. We evaluate EnerGNN on Quadratic Knapsack (QKP), Combinatorial Auctions (CA), and a personalized cell therapy supply chain problem (CART). Our experiments demonstrate the benefits of nonlinear energy modeling for interaction-heavy objectives, the ability to obtain near-optimal solutions while satisfying hard constraints, generalization to larger problem sizes, and zero-shot constraint injection without retraining. The code and data used in this work are available at \url{https://github.com/nikitrian/EnerGNN}.

\begin{figure}[t!]
    \centering
    \makebox[\textwidth][c]{%
    \includegraphics[width=1.05\textwidth]{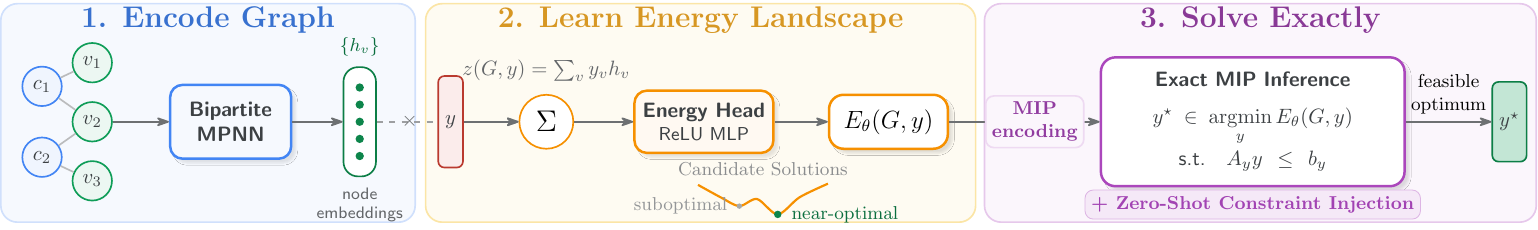}
    }
    \vspace{-0.5cm}
    \caption{EnerGNN encodes an optimization instance using a bipartite message-passing GNN, learns an optimization-compatible energy function over binary configurations, and performs exact constrained inference by globally minimizing the learned energy over the encoded feasible set.}
    \label{fig:architecture}
    \vspace{-6mm}
\end{figure}

\section{Related Work}

\noindent\textbf{Learning-based methods for combinatorial optimization.}
Learning-based approaches to combinatorial optimization commonly use neural models to amortize expensive search across related problem instances. Neural combinatorial optimization methods learn policies that construct solutions sequentially using graph-based or attention-based architectures, often trained through reinforcement learning \citep{khalil2017learning,bello2016neural,kool2018attention,kwon2020pomo}. Another line of work integrates learning into branch-and-bound algorithms, including branching-variable selection \citep{khalil2016learning,gasse2019exact,gupta2020hybrid} and search-tree node selection \citep{he2014learning}. Learning-based predict-and-search methods instead use neural predictions to identify promising variable assignments and solve reduced or restricted optimization problems. Neural Diving predicts distributions over variable assignments and fixes subsets of variables before solving reduced subproblems \citep{nair2020solving}. Predict-and-Search (PaS) uses GNN predictions to define search neighborhoods around predicted solutions \citep{han2023gnn}, while Contrastive Predict-and-Search (ConPaS) extends this approach with contrastive learning to improve solution prediction \citep{huang2024contrastive}.

\noindent\textbf{Differentiable optimization and constraint enforcement.}
Decision-focused learning trains predictive models using downstream decision quality, while differentiable optimization methods integrate optimization problems into neural networks and propagate learning signals through their solutions \citep{amos2017optnet,agrawal2019differentiable,ferber2020mipaal,paulus2021comboptnet, parmentier2025learning}. These approaches enable learning of optimization problem parameters, including objective coefficients and, in some formulations, constraints. Related methods enforce feasibility of neural predictions through differentiable correction, projection, or constraint-enforcement layers \citep{donti2021dc,min2024hardnet,wang2023linsatnet,tordesillas2023rayen,liang2024homeomorphic}. Unlike differentiable optimization approaches, EnerGNN uses the MIP solver for inference without differentiating through it during training. 

 \noindent\textbf{Energy-based structured prediction.}
Structured Prediction Energy Networks (SPENs) learn expressive energy functions over structured outputs and perform inference through continuous relaxation and gradient-based optimization \citep{belanger2016structured}. Input Convex Neural Networks (ICNNs) impose convexity with respect to the optimized inputs, enabling globally optimal inference in continuous convex settings \citep{amos2017input}. Graph Structured Prediction Energy Networks combine explicit local structure with learned higher-order interactions and perform iterative inference over a constrained relaxation \citep{graber2019graph}. \citet{blondel2022learning} develop
generalized Fenchel--Young losses for training energy
networks without differentiating through the inference
optimization problem. Energy-based models have also been integrated into decision-focused stochastic optimization, where learned energy functions are used to approximate optimization landscapes \citep{kong2022end}. More recently, \citet{huang2026unsupervised}
propose an energy-based Langevin sampler for mixed integer
optimization, combining neural sampling of integer
variables with linear programming for continuous variables. EnerGNN builds on energy-based structured prediction but designs its learned energy function to admit an exact MIP representation, enabling global minimization over binary configurations under explicit hard constraints without requiring convexity.

\noindent\textbf{Mixed integer representations of neural networks.}
Piecewise linear neural networks, particularly ReLU networks, can be represented exactly using mixed integer linear constraints \citep{tjeng2017evaluating,fischetti2018deep,anderson2020strong,tsay2021partition}. These formulations have been used for neural network verification and for embedding trained neural surrogate models inside optimization problems. Bound-tightening techniques can improve the strength and computational performance of these formulations \citep{grimstad2019relu}. EnerGNN uses these formulations to optimize a learned energy over binary configurations, with the GNN encoder evaluated before optimization and the readout and ReLU energy head represented exactly within the MIP formulation.

\section{Methodology}
\label{sec:method}
We present the EnerGNN framework, which combines energy learning over binary configurations with exact constrained inference through mixed integer optimization. We first formulate the learning and inference problems, then describe the graph-based energy architecture, the training objective, and its exact MIP encoding. Finally, we establish the theoretical properties of the resulting framework.

\enlargethispage{\baselineskip}

\subsection{Problem Setting}

We consider mixed integer optimization problems of the form
\begin{equation}
\begin{aligned}
    \min_{\mathbf{y},\mathbf{x}} \quad
    & f(\mathbf{y},\mathbf{x}) \\
    \text{s.t.} \quad
    & \mathbf{A}\mathbf{y}+\mathbf{B}\mathbf{x}
      \le \mathbf{b},\\
    & \mathbf{y}\in\{0,1\}^{n}.
\end{aligned}
\label{eq:original_problem}
\end{equation}
where $y$ denotes binary decision variables and $x$ denotes any remaining
operational variables.

We consider two settings. In direct decision problems, such as Quadratic
Knapsack (QKP) and Combinatorial Auctions (CA), $\mathbf{y}$ constitutes the full
decision vector. In hierarchical decomposition problems, such as the cell
therapy supply chain problem (CART), $\mathbf{y}$ represents high-level planning decisions whose assignments define a reduced downstream scheduling problem. In the latter setting, fixing $\mathbf{y}$ induces the
value function
\begin{equation}
F(\mathbf{y})
=
\min_{\mathbf{x}}
\left\{
f(\mathbf{y},\mathbf{x})
\;\middle|\;
\mathbf{A}\mathbf{y}+\mathbf{B}\mathbf{x}
\le \mathbf{b}
\right\}.
\label{eq:value_function}
\end{equation}
EnerGNN does not attempt to reproduce the original objective analytically.
Instead, it learns an energy function $E_\theta(\mathcal{G},\mathbf{y})$ over assignments of the binary decision variables, referred to as configurations, where $\mathcal{G}$ is a graph representation of the problem instance. Training uses solver-generated candidate configurations and their associated solution quality to learn an ordering in which lower energies correspond to higher-quality configurations. For direct decision problems, this ordering reflects the original objective, while for hierarchical decision problems it approximates the ordering induced by $F(\mathbf{y})$.

At inference, the learned energy is minimized subject to explicit hard
constraints:
\begin{equation}
\begin{aligned}
\mathbf{y}^\star \in
\arg\min_{\mathbf{y}\in\{0,1\}^{n}}
\quad & E_\theta(\mathcal{G},\mathbf{y}) \\
\text{s.t.} \quad
& \mathbf{A}_y\mathbf{y}\le\mathbf{b}_y.
\end{aligned}
\label{eq:neural_mip}
\end{equation}
The constraints define the feasible set, while the learned energy
defines the objective optimized over that set. This formulation
defines Neural-MIP inference.

\subsection{Architecture Design}

\paragraph{Graph Neural Network Encoder.}
Each instance is represented as a bipartite graph
$\mathcal{G} = (\mathcal{C} \cup \mathcal{V}, \mathcal{E})$,
where $\mathcal{V}$ denotes the binary decision variables and
$\mathcal{C}$ represents constraints or interactions, depending on
the problem structure. Following the variable-constraint graph
representation commonly used for learning in MIP solvers
\citep{gasse2019exact}, a bipartite message-passing GNN produces
node embeddings $\{\mathbf{h}_v\}_{v \in \mathcal{V}}$ and
$\{\mathbf{h}_c\}_{c \in \mathcal{C}}$, together with a global graph
embedding $\mathbf{h}_{\mathcal{G}}$. The encoder is evaluated once
per instance, and its outputs remain fixed during optimization.

\paragraph{Optimization-Compatible Readout.}
To enable exact optimization, the mapping
$\mathbf{y} \mapsto \mathbf{z}^0$ must admit a mixed integer linear
representation. We construct a benchmark-specific readout that
combines the embeddings of selected binary variables with additional
structural features:
\begin{equation}
    \mathbf{z}^0(\mathbf{y}) =
    \Bigl[
        \sum_{v \in \mathcal{V}} y_v \mathbf{h}_v,\;
        \mathbf{z}_{\mathrm{lin}}(\mathbf{y}),\;
        \mathbf{h}_{\mathcal{G}}
    \Bigr].
    \label{eq:readout}
\end{equation}
The additional features $\mathbf{z}_{\mathrm{lin}}(\mathbf{y})$ capture problem-specific information, such as unselected-item embeddings and capacity utilization in QKP, or occupied-item embeddings in CA. The global graph embedding $\mathbf{h}_{\mathcal{G}}$ provides an instance-level representation that remains fixed across binary configurations. All decision-dependent readout components are either affine in $\mathbf{y}$ or admit exact linearization using auxiliary variables. Benchmark-specific graph representations and readout definitions are provided in Appendix~\ref{app:benchmarks}.

\paragraph{Energy Head.}
A multilayer perceptron (MLP) with ReLU activations maps
$\mathbf{z}^0(\mathbf{y})$ to a scalar energy
$E_\theta(\mathcal{G},\mathbf{y})$.
The nonlinear energy head enables the model to capture interactions
between binary decisions, while its piecewise linear structure,
combined with the optimization-compatible readout, admits an exact
mixed integer linear representation.

\subsection{Optimality-Aware Energy Learning}

We train on candidate configurations $\mathcal{Y}=\mathcal{F}\cup\mathcal{I}$,
containing feasible and infeasible configurations. Each feasible candidate $\mathbf{y}_i$ is assigned a relative primal gap $g_i$, measured against the best available reference solution for that instance, while each infeasible candidate $\mathbf{y}_k$ is associated with a normalized constraint violation magnitude $\nu_k$. The loss combines an optimality-aware contrastive term, feasible calibration, and infeasibility shaping:
\begin{equation}
    \mathcal{L}
    =
    \mathcal{L}_{\mathrm{con}}
    +
    \lambda_{\mathrm{cal}} \mathcal{L}_{\mathrm{cal}}
    +
    \lambda_{\mathrm{inf}} \mathcal{L}_{\mathrm{inf}}.
\end{equation}

\paragraph{Optimality-aware contrastive loss.}
 For each instance, we construct a target distribution that assigns greater probability to feasible configurations with smaller primal gaps and zero probability to infeasible configurations. The learned energy induces a second distribution over the entire candidate pool, and the Kullback–Leibler divergence encourages the two distributions to agree.
\begin{equation}
\begin{aligned}
    \mathcal{L}_{\mathrm{con}}
    &= \mathrm{KL}(p^\star \,\|\, p_\theta), \\[0.3em]
    p_i^*
    &=
    \begin{cases}
    \dfrac{\exp(-g_i/\tau_{\mathrm{rank}})}
    {\sum_{j \in \mathcal{F}} \exp(-g_j/\tau_{\mathrm{rank}})},
    & i \in \mathcal{F}, \\[0.8em]
    0, & i \notin \mathcal{F},
    \end{cases}
    \\[0.3em]
    p_{\theta,i}
    &=
    \dfrac{\exp(-E_\theta(\mathcal{G},\mathbf{y}_i)/T)}
    {\sum_{j \in \mathcal{Y}} \exp(-E_\theta(\mathcal{G},\mathbf{y}_j)/T)},
    \qquad i \in \mathcal{Y}.
\end{aligned}
\end{equation}
where $T$ is the model softmax temperature and $\tau_{\mathrm{rank}}$ controls the target sharpness.

\paragraph{Calibration loss.}
The contrastive loss primarily supervises relative preferences within the candidate pool. We additionally calibrate the energies of feasible configurations against their relative primal gaps using Huber regression with an instance-adaptive transition scale.
\begin{equation}
    \mathcal{L}_{\mathrm{cal}}
    =
    \frac{1}{|\mathcal{F}|}
    \sum_{i \in \mathcal{F}}
    H_{\delta_{\mathcal{F}}}\!\left(
    E_\theta(\mathcal{G}, \mathbf{y}_i) - g_i
    \right).
\end{equation}
The transition scale $\delta_{\mathcal{F}}$ is computed from the primal-gap
span of feasible candidates. Details are given in Appendix~\ref{app:loss_details}.

\paragraph{Infeasibility loss.}
For infeasible configurations, we use an auxiliary Huber loss to align their energies with a logarithmically scaled measure of constraint violation:
\begin{equation}
    \mathcal{L}_{\mathrm{inf}}
    =
    \frac{1}{|\mathcal{I}|}
    \sum_{k \in \mathcal{I}}
    H_{\delta_{\mathrm{inf}}}\!\left(
    E_\theta(\mathcal{G}, \mathbf{y}_k)
    -
    \log(1+\nu_k)
    \right),
    \label{eq:infeasibility_loss}
\end{equation}
where $\nu_k \ge 0$ is the normalized constraint violation. This term provides a graded training signal within the infeasible region, while $\mathcal{L}_{\mathrm{con}}$ discourages infeasible candidates from receiving probability mass at the expense of high-quality feasible configurations. Feasibility at inference is enforced explicitly by the MIP constraints.

\subsection{Inference via Neural-MIP Encoding}

At inference, the trained GNN encoder is evaluated once, while the decision-dependent readout and ReLU energy head are represented exactly within the MIP using auxiliary variables and linear constraints. Valid pre-activation bounds are obtained through interval bound propagation using benchmark-specific input bounds \citep{gowal2018effectiveness}. Neurons with fixed activation states are eliminated to reduce the formulation size. The exact Neural-MIP formulation is detailed in Appendix~\ref{app:mip_encoding}, with the bound-tightening procedure described in Appendix~\ref{app:ibp}.

\subsection{Theoretical Properties}
\label{sec:theory}

We summarize the formal properties induced by the architectural design. 
Detailed proofs are provided in Appendix~\ref{app:theory}.

\begin{proposition}[MIP-representable energy]
\label{prop:representability}
For a fixed instance graph $\mathcal{G}$, assume that the encoder is evaluated
once, the readout $\mathbf{z}^0(\mathbf{y})$ admits an exact mixed integer
linear representation over a bounded domain, and the energy head is a
feed-forward ReLU network. Then
$\mathbf{y}\mapsto E_\theta(\mathcal{G},\mathbf{y})$ is exactly
MIP-representable on that domain.
\end{proposition}

\begin{proof}[Proof sketch]
For fixed $\mathcal{G}$, encoder outputs are constants. The readout admits an
exact mixed integer linear representation by construction. A ReLU network is a
composition of affine maps and ReLU operators, each of which admits an exact
mixed integer linear encoding over bounded input domains. Hence the full energy
admits an exact MIP representation.
\end{proof}

\begin{corollary}[Exact feasible inference]
\label{cor:feasible_inference}
If the encoded feasible set
\[
\mathcal{C}
=
\left\{
\mathbf{y}\in\{0,1\}^{n}
\mid
\mathbf{A}_y\mathbf{y}\le \mathbf{b}_y
\right\}
\]
is nonempty, then solving
\[
\min_{\mathbf{y}\in\mathcal{C}} E_\theta(\mathcal{G},\mathbf{y})
\]
is equivalent to solving a mixed integer linear program and returns a feasible
solution that is globally optimal with respect to the learned energy.
\end{corollary}

\begin{corollary}[Constraint injection]
\label{cor:constraint_injection}
The same learned energy can be optimized over any bounded test-time feasible
set obtained by adding, removing, or modifying linear constraints over the same
decision variables, without retraining the EBM. If the test-time feasible set
is not contained in the domain used to compute the original Big-M bounds, the
bounds used in the MIP encoding must be recomputed for the new feasible set.
\end{corollary}

\begin{proposition}[Expressivity beyond factorized scoring]
If the readout contains the binary variables $\mathbf y$, or an affine
representation from which they can be recovered, a one-hidden-layer ReLU
energy can represent arbitrary quadratic pairwise interactions over binary
variables exactly, whereas a factorized linear score cannot represent a
nonzero pairwise interaction.
\end{proposition}

\section{Experiments}
\label{sec:experiments}

We evaluate EnerGNN on three combinatorial optimization problems: Quadratic Knapsack (QKP), Combinatorial Auctions (CA), and a Chimeric Antigen Receptor (CAR) T-cell therapy supply chain (CART). QKP and CA evaluate direct decision inference, while CART considers hierarchical decomposition, where EnerGNN selects high-level capacitated facility location decisions before solving a reduced downstream patient scheduling problem. Our experiments assess solution quality, feasibility, the contribution of nonlinear energy modeling, cross-scale generalization, and zero-shot constraint injection. Each benchmark examines a different aspect of the framework: QKP tests nonlinear interactions between variables, CA evaluates inference under packing constraints, and CART examines learning over complicating variables in hierarchical optimization.

\subsection{Benchmarks}

\noindent\textbf{Quadratic Knapsack.}
QKP is an interaction-heavy problem involving the selection of items under a capacity constraint, with both individual profits and pairwise interactions contributing to the objective. We consider instances with 150 binary item-selection variables and an interaction density of 0.30.

\noindent\textbf{Cell Therapy Supply Chain.}
CART is a time-indexed, capacitated facility location and patient allocation
MIP for personalized CAR T-cell therapy supply chains. The network comprises
16 candidate manufacturing facilities. The formulation incorporates
2 transportation modes and daily scheduling over a 130-day horizon,
with patient demand varying from 10 to 650 across different demand
distributions. EnerGNN optimizes the 16 first-stage facility location
variables, with each configuration of these variables inducing a tactical patient scheduling MIP.

\noindent\textbf{Combinatorial Auctions.}
CA is a packing problem with 175 binary bid-selection variables, extended to 250 in the cross-scale evaluation. Each bid consumes a subset of items, and each item can be allocated at most once. Because the original objective is linear and these instances are comparatively easy for modern MIP solvers, CA primarily evaluates constrained inference and generalization rather than computational speedup.

Detailed benchmark problem formulations are given in
Appendix~\ref{app:benchmarks}. Instance generation procedures and dataset
statistics are described in Appendix~\ref{app:data}. Benchmark-specific
Neural-MIP input bounds and hard feasibility constraints are detailed in
Appendices~\ref{app:input_bounds} and~\ref{app:hard_constraints},
respectively.

\subsection{Baselines and Metrics}

We compare against three classes of baselines that differ in how learning is
combined with optimization. \emph{Full MIP} solves the original problem using
branch-and-cut and provides a reference optimum.
\emph{Predict+Search}, \emph{Neural Diving}~\citep{nair2020solving}, and
\emph{ConPaS}~\citep{huang2024contrastive} use learned predictions to restrict
or guide the solver's search. Predict+Search fixes variables based on prediction
confidence and solves the resulting reduced problem. Neural Diving samples
variable assignments, fixes subsets of variables, and solves multiple reduced
subproblems. ConPaS instead imposes a neighborhood constraint around its
learned prediction while retaining the original objective.

\emph{Linear+MIP} and \emph{EnerGNN} learn objectives that are optimized
directly under explicit hard constraints. Linear+MIP uses additive
per-variable scores, whereas EnerGNN learns a joint energy over complete
configurations that can capture interactions between decisions. We evaluate
Linear+MIP with binary cross-entropy (BCE) and contrastive supervision,
and include \emph{EnerGNN (Linear)} as an ablation that replaces the ReLU energy head with a linear head over the same readout. This isolates the contribution of the nonlinear energy head. All learning-based methods use
the same bipartite GNN encoder and input features where applicable, but
differ in their prediction heads, training objectives, and inference
procedures. Full implementation details and inference settings are
provided in Appendix~\ref{app:baselines}.

We report mean primal gap, runtime, feasibility, and the fraction of instances with a primal gap of at most $1\%$ relative to the certified global optimum. Primal gaps
are computed over all test instances, assigning a $100\%$ penalty
to infeasible outputs. EnerGNN achieves $100\%$ feasibility across
the reported experiments, with hard constraints enforced directly
during inference.

\begin{table*}[h]
\setlength{\belowcaptionskip}{1pt}
\caption{
Main results across benchmarks (mean over test instances). We report primal gap,
runtime, feasibility, and the fraction of instances achieving either exact
recovery or a solution within $1\%$ of the global optimum.}
\centering
\small
\setlength{\tabcolsep}{2pt}
\resizebox{\textwidth}{!}{%
\begin{tabular}{l|cccc|cccc|cccc}
\toprule
& \multicolumn{4}{c|}{\textbf{QKP}} 
& \multicolumn{4}{c|}{\textbf{CART}} 
& \multicolumn{4}{c}{\textbf{CA}} \\
\cmidrule(lr){2-5} \cmidrule(lr){6-9} \cmidrule(lr){10-13}
\textbf{Method} 
& \makecell{Gap\\(\%)} & \makecell{Time\\(s)} & \makecell{$\leq$1\%\\(\%)} & \makecell{Feas.\\(\%)}
& \makecell{Gap\\(\%)} & \makecell{Time\\(s)} & \makecell{$\leq$1\%\\(\%)}  & \makecell{Feas.\\(\%)}
& \makecell{Gap \\(\%)} & \makecell{Time\\(s)} & \makecell{$\leq$1\%\\(\%)} & \makecell{Feas.\\(\%)} \\
\midrule
Full MIP (Gurobi)
& 0.00 & 525.21 & 100 & 100
& 0.00 & 3728.82 & 100 & 100
& 0.00 & 0.14 & 100 & 100 \\

Predict+Search
& 1.13 & 0.12 & 97 & 99
& 22.02 & 251.79 & 21 & 83
& 29.81 & 0.01 & 0 & 100 \\

Neural Diving 
& 1.04 & 1.52 & 99 & 99 
& 1.53 & 398.34 & 65 & 100 
& 2.91 & 1.23 & 5 & 100 \\ 

ConPaS 
& 0.95 & 1.03 & 100 & 100 
& 0.02 & 3404.19 & 100 & 100 
& 0.90 & 0.26 & 65 & 100 \\

Linear+MIP (BCE)
& 4.92 & 0.03 & 20.1 & 100
& 15.81 & 273.60 & 17 & 100
& 83.99 & 0.02 & 0 & 100 \\

Linear+MIP (Contrastive)
& 58.55 & 0.03 & 0.00 & 100
& 24.18 & 410.04 & 4 & 100
& 30.82 & 0.02 & 0 & 100 \\

EnerGNN (Linear)
& 26.87 & 0.28 & 0.00 & 100
& 21.80 & 391.36 & 5 & 100
& 0.20 & 0.31 & 95 & 100 \\

\midrule
EnerGNN (MLP)
& 0.31 & 3.56 & 99 & 100
& 0.05 & 214.03 & 100 & 100
& 0.53 & 8.62 & 77 & 100 \\
\bottomrule
\end{tabular}}

\label{tab:main_results}
\end{table*}

\subsection{Main Results}

Table~\ref{tab:main_results} summarizes the results across the three benchmarks.

On QKP, EnerGNN (MLP) achieves the lowest mean primal
gap among learning-based methods at 0.31\%, with 99\%
of instances within 1\% of the global optimum and 100\%
feasibility. ConPaS, Neural Diving, and Predict+Search
achieve mean gaps of 0.95\%, 1.04\%, and 1.13\%,
respectively, while the linear scoring baselines perform
substantially worse. The difference between EnerGNN (MLP) and its linear head ablation (0.31\% versus 26.87\%) is particularly informative. Both use the same graph encoder and decision-dependent readout, but the nonlinear head can capture additional interactions between readout features. Predict+Search, Neural Diving, and ConPaS remain faster, illustrating a trade-off between solution quality and inference time.

On CART, the strategic facility location decisions determine a much
larger downstream patient allocation MIP. Neural Diving
and ConPaS use learned predictions to restrict the strategic search,
achieving mean gaps of 1.53\% and 0.02\% in 398.34~s and
3404.19~s, respectively. While Neural Diving uses parallel dives, allowing flips in the ConPaS neighborhood requires retaining all candidate facilities in our CART formulation, preventing the elimination of facility-dependent patient allocation dimensions.
EnerGNN instead learns the downstream value of complete strategic
configurations offline and optimizes this learned energy during
inference. Once the facility configuration is selected, the
operational MIP is constructed using only the selected facilities,
reducing the number of variables and constraints by approximately 77\% (Appendix~\ref{sec:cart_scale_cons}). EnerGNN
achieves a 0.05\% mean gap in 214.03~s, with all test instances
within 1\% of the global optimum and 100\% feasibility, making
it approximately 1.9$\times$ faster than Neural Diving,
15.9$\times$ faster than ConPaS, and 17.4$\times$ faster than
the monolithic MIP.

On CA, the full MIP is solved by Gurobi in 0.14~s, so this benchmark is
not intended to demonstrate solver acceleration. Instead, it tests whether a
learned energy can be optimized under a large number of explicit packing
constraints. The linear EnerGNN head achieves a 0.20\% mean gap with 95\% of
instances within 1\% of the optimum and 100\% feasibility. It outperforms the nonlinear variant, which achieves
a 0.53\% mean gap at substantially higher inference
cost. In contrast to QKP, CA has a linear original objective, and the linear energy head performs better on these instances. These results indicate that the benefit of nonlinear energy heads depends on the problem structure.

Figure~\ref{fig:energy} further examines whether
the learned energy reflects the quality of complete
configurations. The strong rank correlations between
energy and primal gap across QKP, CART, and CA
($\rho \geq 0.94$) indicate that lower energy
configurations generally correspond to better
solutions.

\begin{figure}[t]
    \centering
    \includegraphics[width=0.70\textwidth]{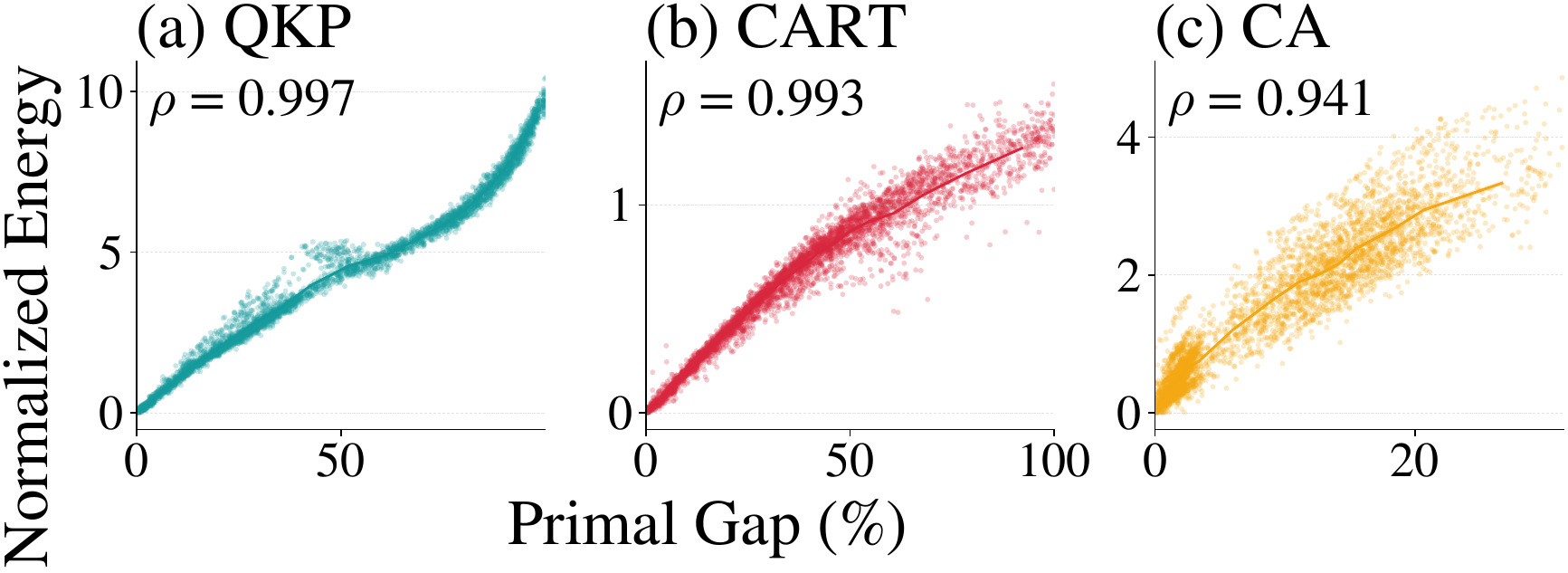}
    \vspace{-0.1cm}
    \caption{
    Energy vs. solution quality. Spearman $\rho$ is reported per benchmark.
    }
    \label{fig:energy}
\vspace{-2mm}
\end{figure}

\begin{table}[h]
\setlength{\abovecaptionskip}{2pt}
\setlength{\belowcaptionskip}{2pt}
\caption{Cross-scale generalization without retraining.}
\centering
\small
\setlength{\tabcolsep}{5pt}
\begin{tabular}{llccc}
\toprule
\textbf{Benchmark}
& \textbf{Train $\rightarrow$ test}
& \makecell{\textbf{Primal Gap }\textbf{(\%)}}
& \makecell{\textbf{Instances $\leq1\%$ }\textbf{(\%)}}
& \makecell{\textbf{Feasibility }\textbf{(\%)}} \\
\midrule
CART
& $m=16 \rightarrow 12$
& 0.98 & 63 & 100 \\

CART
& $m=16 \rightarrow 16$
& 0.05 & 100 & 100 \\

CART
& $m=16 \rightarrow 20$
& 1.23 & 49 & 100 \\

\midrule
CA
& $175 \rightarrow 175$
& 0.20 & 95 & 100 \\

CA
& $175 \rightarrow 250$
& 0.98 & 69 & 100 \\
\bottomrule
\end{tabular}
\label{tab:generalization}
\end{table}

\subsection{Cross-Scale Generalization}

We evaluate whether EnerGNN transfers to different problem sizes without
retraining. In the CART setting, the model is trained on instances with
$m=16$ candidate facilities, and is evaluated at $m=12$ and $m=20$. In CA, the number of
binary bid variables increases from 175 during training to 250 at test time. Table~\ref{tab:generalization} shows that EnerGNN maintains 100\% feasibility across all evaluated scales. In CART, the model achieves a mean gap of 0.98\% when transferred to
$m=12$ facilities and 1.23\% when extrapolated to $m=20$, compared with
0.05\% at the training scale. For CA, we report the linear head variant of EnerGNN, which achieves the best performance on this benchmark. Results for the nonlinear energy head are provided in Appendix~\ref{app:cross_scale_ca}. Increasing the number of bids from 175 to 250 increases the mean primal gap from 0.20\% to 0.98\%. These results show that the learned energy can be
applied to variable spaces of different sizes without retraining, although
solution quality degrades under larger extrapolation shifts.

\begin{table*}[t]
\centering
\caption{Zero-shot constraint injection.
We report the fraction of originally selected variables removed, the mean Hamming
shift from the original optimum, mean primal gap, the fraction of instances within
$1\%$ of the modified global optimum, and feasibility.}
\label{tab:constraint_injection}
\small
\setlength{\tabcolsep}{2.2pt}
\begin{tabular}{llcc|ccc|ccc|ccc}
\toprule
& & & &
\multicolumn{3}{c|}{\textbf{EnerGNN}} &
\multicolumn{3}{c|}{\textbf{Neural Diving}} &
\multicolumn{3}{c}{\textbf{ConPaS}} \\
\cmidrule(lr){5-7}
\cmidrule(lr){8-10}
\cmidrule(lr){11-13}

\textbf{Benchmark}
& \makecell{\textbf{Constraint}\\\textbf{shift}}
& \makecell{Selected\\removed (\%)}
& \makecell{Hamming\\shift}
& \makecell{Gap\\(\%)}
& \makecell{$\leq 1\%$\\(\%)}
& \makecell{Feas.\\(\%)}
& \makecell{Gap\\(\%)}
& \makecell{$\leq 1\%$\\(\%)}
& \makecell{Feas.\\(\%)}
& \makecell{Gap\\(\%)}
& \makecell{$\leq 1\%$\\(\%)}
& \makecell{Feas.\\(\%)} \\
\midrule

\makecell[l]{CART\\(facilities)}
& Remove 1
& 31.35 & 3.84
& 0.21 & 96.24 & 100
& 2.99 & 40.44 & 89.66
& 0.30 & 87.15 & 100 \\

& Remove 2
& 50.81 & 5.27
& 0.34 & 93.27 & 100
& 4.38 & 25.06 & 72.85
& 0.70 & 77.96 & 99.77 \\

& Remove 3
& 61.99 & 6.26
& 0.45 & 91.23 & 100
& 4.32 & 20.18 & 49.12
& 1.08 & 69.88 & 98.28 \\

\midrule

\makecell[l]{CA\\(bids)}
& Remove 5
& 20.30 & 34.87
& 0.21 & 93.36 & 100
& 3.41 & 9.73 & 100
& 1.13 & 54.09 & 100 \\

& Remove 10
& 40.61 & 38.77
& 0.17 & 95.27 & 100
& 4.12 & 5.09 & 100
& 1.45 & 41.09 & 100 \\

& Remove 15
& 60.91 & 41.36
& 0.16 & 95.73 & 100
& 4.76 & 3.82 & 100
& 1.62 & 38.27 & 100 \\

& Force 1
& -- & 31.25
& 0.07 & 99.00 & 100
& 3.60 & 17.00 & 100
& 1.06 & 57.00 & 100 \\

\bottomrule
\end{tabular}
\vspace{-2mm}
\end{table*}

\subsection{Zero-Shot Constraint Injection}

We evaluate whether a trained EnerGNN can be reused after the feasible region changes without retraining. The learned energy is held fixed, and only the constraints of the inference MIP are modified. In CART ($m=16$), we deactivate one, two, or three facilities selected in the original optimum. In CA, we remove 5, 10, or 15 bids selected in the original optimal allocation, corresponding to approximately 20\%, 40\%, and 60\% of the selected bids, respectively, and additionally test a setting in which the highest-value bid rejected by the original optimal solution is forced into the allocation. The CART evaluation comprises 1,092 perturbations in total (319, 431, and 342 cases for removing one, two, and three selected facilities, respectively). For CA, each bid-removal setting is evaluated over 1,100 perturbations, while the forced-bid setting contains 100 perturbations. For CA constraint-injection experiments, we use the linear EnerGNN variant, which performs best on the in-distribution CA benchmark.

Table~\ref{tab:constraint_injection} compares EnerGNN with Neural Diving and ConPaS under the same constraint shifts. EnerGNN remains feasible in every case and achieves the lowest mean gap across all tested perturbations. On CART, its mean gap increases from 0.21\% to 0.45\% as one to three originally selected facilities are deactivated, while maintaining at least 91.23\% of instances within 1\% of the modified global optimum. By contrast, Neural Diving degrades substantially under the same shifts, reaching only 49.12\% feasibility when three facilities are deactivated, while ConPaS exhibits larger gaps than EnerGNN across all CART perturbations. On CA, EnerGNN achieves mean gaps between 0.07\% and 0.21\%, with at least 93.36\% of instances within 1\% of the modified optimum. Both baselines obtain substantially larger gaps across all CA constraint shifts. The perturbations induce substantial changes in the selected solutions, with a mean Hamming shift of 6.26 when three CART facilities are deactivated and 41.36 when 15 CA bids are removed. Additional constraint-injection experiments are reported in Appendix~\ref{app:zero_shot}.

\subsection{Loss Ablations}

\begin{table}[H]
\setlength{\abovecaptionskip}{2pt}
\setlength{\belowcaptionskip}{2pt}
\caption{Loss ablation.}
\centering
\small
\setlength{\tabcolsep}{4pt}
\begin{tabular}{lcc|cc|cc}
\toprule
& \multicolumn{2}{c|}{\textbf{CART}} 
& \multicolumn{2}{c|}{\textbf{QKP}} 
& \multicolumn{2}{c}{\textbf{CA}} \\
\cmidrule(lr){2-3}\cmidrule(lr){4-5}\cmidrule(lr){6-7}
\textbf{Variant}
& Gap (\%) & $\leq$1\% (\%)
& Gap (\%) & $\leq$1\% (\%)
& Gap (\%) & $\leq$1\% (\%) \\
\midrule
Full EnerGNN      
& 0.05 & 100   
& 0.31 & 99 
& 0.53 & 77 \\

No calibration loss   
& 0.78 & 70  
& 0.88 & 68 
& 7.69 & 4 \\

No infeasibility loss 
& 0.06 & 100
& 0.39 & 99 
& 0.58 & 77 \\
\bottomrule
\end{tabular}

\label{tab:loss_ablation}
\end{table}

We ablate the loss components used to shape the energy landscape. The full loss
combines an optimality-aware contrastive loss, feasible calibration, and
infeasibility shaping. We additionally evaluate an alternative loss enforcing strict feasible--infeasible ordering in Appendix~\ref{app:strict_dominance}.
Table~\ref{tab:loss_ablation} shows that the calibration loss is important for learning
an energy scale that reflects solution quality. Removing calibration increases the mean gap on CART from $0.05\%$
to $0.78\%$ and reduces the fraction of instances within $1\%$ of the global
optimum from $100\%$ to $70\%$. On QKP, the corresponding fraction decreases
from $99\%$ to $68\%$, while the largest degradation occurs on CA, where the
mean gap increases from $0.53\%$ to $7.69\%$ and only $4\%$ of instances remain
within $1\%$ of the optimum. In contrast, removing the infeasibility term has only a limited effect across all three benchmarks, suggesting that its contribution to solution quality is comparatively small.

\subsection{Inference Size and Tractability}

Table~\ref{tab:inference_size} reports the size of the Neural-MIP, where
$|\mathbf y|$ denotes the original binary variables optimized by
EnerGNN. The graph encoder is evaluated once before optimization and therefore
does not contribute optimization variables. Additional binaries arise from
the exact readout linearizations and the compact ReLU energy head.
Appendix~\ref{app:head_arch} shows that increasing the width or depth of the
energy head substantially increases inference time without
improving solution quality.

\begin{table*}[!htbp]
\centering
\small
\setlength{\tabcolsep}{4pt}
\caption{Neural-MIP formulation dimensions and mean inference time
across test instances. For CART, inference time excludes
downstream patient scheduling.}
\label{tab:inference_size}
\begin{tabular}{lcccccccc}
\toprule
\textbf{Benchmark}
& $|\mathbf y|$
& \makecell{\textbf{Readout}\\\textbf{binaries}}
& \makecell{\textbf{ReLU}\\\textbf{binaries}}
& \makecell{\textbf{Total}\\\textbf{binaries}}
& \makecell{\textbf{Continuous}\\\textbf{variables}}
& \makecell{\textbf{Hard feasibility}\\\textbf{constraints}}
& \makecell{\textbf{Total}\\\textbf{constraints}}
& \makecell{\textbf{Inference}\\\textbf{time (s)}} \\
\midrule
CART & 16  & 4 & 32 & 52  & 290 & 1      & 461    & 1.46  \\
QKP    & 150 & 0 & 32 & 182 & 292 & 1      & 389    & 3.56  \\
CA     & 175 & 0 & 32 & 207 & 298 & 131.74 & 525.74 & 8.62 \\
\bottomrule
\end{tabular}

\end{table*}

\subsection{Limitations}
EnerGNN requires the decision-dependent energy head to remain MIP-representable, creating an expressiveness--tractability trade-off in the embedded component. In our experiments, compact shallow heads are sufficient, while increasing width or depth can substantially increase inference time without systematically improving solution quality. Additional representational capacity can instead be placed in the GNN encoder, which is evaluated before optimization and does not increase Neural-MIP size. Training also relies on solver-generated candidate configurations and labels,
introducing a one-time offline cost that is most appropriate for repeated-deployment
settings; a complete cost and amortization analysis is provided in
Appendix~\ref{app:ch5_amortisation}. Finally, exact inference is with respect to the learned energy rather than the original objective. Constraint injection guarantees feasibility with respect to the encoded constraints, but solution quality may degrade under sufficiently large distribution or feasible-set shifts.

\section{Conclusion}

We introduced EnerGNN, an energy-based framework that learns optimization-compatible energy functions over binary configurations and minimizes them under explicit hard constraints through mixed integer optimization. The framework combines a graph-based encoder with a MIP-representable readout and nonlinear energy head, enabling exact optimization of the learned energy.

Our experiments demonstrate the framework across direct decision and
hierarchical decomposition settings. EnerGNN achieves near-optimal solutions
with 100\% feasibility across Quadratic Knapsack, Combinatorial Auctions,
and a personalized cell therapy supply chain problem. On the latter, EnerGNN achieves substantial computational speedups
over the monolithic MIP and learning-based baselines while
maintaining near-optimal solution quality. EnerGNN also generalizes to larger problem instances
without retraining and supports zero-shot constraint injection, allowing
the same learned energy to be optimized under changing operational
constraints while preserving feasibility. This adaptability is
particularly valuable in deployment settings where operational
requirements and resource availability may change over time.

More broadly, EnerGNN suggests a design principle for reliable AI systems.
When discrete decisions must satisfy hard constraints, learning can shape
objectives that algorithmic modules optimize rather than producing final
decisions directly. This shifts the learned object from the decision itself to
an optimization-compatible objective over decisions, allowing learned
preferences and explicit constraints to remain separate at inference.

{\small
\bibliographystyle{unsrtnat}
\bibliography{mybibliography}
}

\newpage

\appendix

\section{Proofs of Theoretical Properties}
\label{app:theory}

We provide the proofs for the theoretical properties stated in
Section~\ref{sec:theory}. Throughout this appendix, the instance graph
$\mathcal{G}$ is fixed. The graph encoder is evaluated once before inference,
and its outputs are treated as constants. The optimization variables in
Neural-MIP inference are therefore the original binary decision variables,
the readout auxiliary variables, the ReLU activation variables, and the scalar
energy variable.

\subsection{Preliminaries}

\begin{definition}[MIP-representability]
\label{def:mip_representable}
Let $\mathcal{D}\subset\mathbb{R}^d$ be a bounded domain. A function
$f:\mathcal{D}\to\mathbb{R}$ is MIP-representable on $\mathcal{D}$ if its graph
restricted to $\mathcal{D}$ can be described exactly by a finite collection of
linear constraints involving auxiliary continuous and integer variables. Equivalently, there
exists a mixed integer linear feasible set whose projection onto the
input-output coordinates is exactly
\[
    \left\{(\mathbf{x},t) \mid t=f(\mathbf{x})\right\}.
\]
\end{definition}

We use the following standard exact representation of a ReLU unit.

\begin{lemma}[Exact ReLU encoding]
\label{lem:relu}
Let $h=\max\{0,a\}$, where the pre-activation $a$ has finite bounds
\[
    \underline{a} \le a \le \overline{a}.
\]
Assume $\underline{a}\le 0 \le \overline{a}$. Then the graph of the ReLU
can be represented exactly using one binary variable $z\in\{0,1\}$ and the
linear constraints
\begin{align}
    h &\ge a, \label{eq:relu_lb_a}\\
    h &\ge 0, \label{eq:relu_lb_zero}\\
    h &\le a-\underline{a}(1-z), \label{eq:relu_ub_a}\\
    h &\le \overline{a}z. \label{eq:relu_ub_zero}
\end{align}
If the bounds imply that the ReLU is always active or always inactive, the
binary variable can be removed and the ReLU can be replaced by a linear
equality.
\end{lemma}

\begin{proof}
Consider first the case $z=0$. Constraint~\eqref{eq:relu_ub_zero} gives
$h\le 0$, and constraint~\eqref{eq:relu_lb_zero} gives $h\ge 0$. Hence
$h=0$. Constraint~\eqref{eq:relu_lb_a} then implies $a\le 0$. Therefore
the inactive case represents $h=\max\{0,a\}=0$.

Now consider the case $z=1$. Constraint~\eqref{eq:relu_ub_a} gives
$h\le a$, and constraint~\eqref{eq:relu_lb_a} gives $h\ge a$. Hence
$h=a$. Constraint~\eqref{eq:relu_lb_zero} also implies $a\ge 0$.
Therefore the active case represents $h=\max\{0,a\}=a$.

The two cases together describe exactly the graph of the ReLU over the
bounded interval $[\underline{a},\overline{a}]$.
\end{proof}

\subsection{Proof of Proposition~\ref{prop:representability}}

\begin{proof}
Fix an instance graph $\mathcal{G}$. The graph encoder produces node
embeddings and a global graph embedding. Since the encoder is evaluated
before inference, these embeddings are constants inside the optimization
problem.

The readout maps the binary decision vector $\mathbf{y}$ to an input feature
vector $\mathbf{z}^0(\mathbf{y})$ for the energy head. By construction, each
component of the readout is either affine in $\mathbf{y}$ or admits an exact
mixed integer linear representation using auxiliary variables. The graph of
the readout can therefore be described by linear constraints together with
integrality restrictions.

The energy head is a feed-forward ReLU network. Each layer applies an affine
transformation followed by componentwise ReLU activations. The affine
transformations are represented by linear equalities. Each ReLU activation is
represented exactly by Lemma~\ref{lem:relu}, using valid lower and upper bounds
on the corresponding pre-activation.

Combining the readout representation, the affine layer equations, the ReLU
encodings, and the final linear output equation gives a finite mixed integer
linear representation of the graph of
\[
    \mathbf{y} \mapsto E_\theta(\mathcal{G},\mathbf{y}).
\]
Thus, for fixed $\mathcal{G}$, the learned energy is MIP-representable.
\end{proof}

\subsection{Proof of Corollary~\ref{cor:feasible_inference}}

\begin{proof}
By Proposition~\ref{prop:representability}, the learned energy
$E_\theta(\mathcal{G},\mathbf{y})$ admits an exact mixed integer linear
representation. The feasible set
\[
    \mathcal{C}
    =
    \left\{
    \mathbf{y}\in\{0,1\}^{|\mathcal{V}|}
    \mid
    \mathbf{A}_y\mathbf{y}\le \mathbf{b}_y
    \right\}
\]
is described by linear constraints and binary restrictions.

The inference problem can therefore be written as the mixed integer linear
optimization problem
\[
\begin{aligned}
    \min_{\mathbf{y},\boldsymbol{\xi},e}
    \quad & e \\
    \text{s.t.}\quad
    & e = E_\theta(\mathcal{G},\mathbf{y}), \\
    & \mathbf{A}_y\mathbf{y}\le \mathbf{b}_y, \\
    & \mathbf{y}\in\{0,1\}^{|\mathcal{V}|},
\end{aligned}
\]
where $\boldsymbol{\xi}$ denotes all auxiliary readout and ReLU variables.

If $\mathcal{C}$ is nonempty and this optimization problem is solved to
global optimality, the returned solution is feasible by construction and
globally optimal with respect to the learned energy over $\mathcal{C}$.
If the solver is stopped early, any returned incumbent remains feasible,
although it need not be globally optimal with respect to the learned energy.
\end{proof}

\subsection{Proof of Corollary~\ref{cor:constraint_injection}}

\begin{proof}
The learned parameters $\theta$ define the energy function
$E_\theta(\mathcal{G},\mathbf{y})$. This function does not depend on the
particular deployment-time feasible region. The feasible region enters only
through the constraints imposed during inference.

Let
\[
    \mathcal{C}_{\mathrm{test}}
    =
    \left\{
    \mathbf{y}\in\{0,1\}^{|\mathcal{V}|}
    \mid
    \mathbf{A}_{\mathrm{test}}\mathbf{y}
    \le
    \mathbf{b}_{\mathrm{test}}
    \right\}
\]
be a test-time feasible set obtained by adding, removing, or modifying linear
constraints over the same decision variables. The same learned energy can be
optimized over the modified feasible set by solving
\[
    \mathbf{y}^\star
    \in
    \arg\min_{\mathbf{y}\in\mathcal{C}_{\mathrm{test}}}
    E_\theta(\mathcal{G},\mathbf{y}).
\]

Exactness of the mixed integer encoding requires valid finite bounds for all
ReLU pre-activations over the feasible set used at inference. If
$\mathcal{C}_{\mathrm{test}}$ is contained in the domain used to compute the
original Big-M constants, the original bounds remain valid. If the feasible set
is expanded or otherwise leaves that domain, the Big-M constants are recomputed
by bound propagation over $\mathcal{C}_{\mathrm{test}}$. This changes only the
encoding constants and does not require retraining the neural model.
\end{proof}

\subsection{Bound Tightening}

The exact ReLU encoding requires valid lower and upper bounds on each
pre-activation. In practice, these bounds are obtained using interval bound
propagation through the assembled readout and energy head. Tighter bounds make
the mixed integer linear representation stronger by reducing Big-M constants
and removing ReLU binaries whose activation status is fixed. The bound
computation used in the experiments is described in
Appendix~\ref{app:ibp}.

\subsection{Expressivity Beyond Factorized Linear Scoring}

We next formalize why the nonlinear energy head can represent interactions
that are impossible for additive per-variable scoring functions.

\begin{definition}[Factorized linear score]
\label{def:factorized_score}
A score over binary configurations is \emph{factorized linear} if it decomposes
additively over individual variables as
\[
    L(\mathbf{y}) = b+\sum_i w_i y_i.
\]
Such a score assigns a contribution to each selected variable independently
of which other variables are selected.
\end{definition}

\begin{proposition}[Expressivity beyond linear scoring]
\label{prop:quadratic_expressivity}
Let
\[
Q(\mathbf{y})
=
c+\sum_{i=1}^n a_i y_i+\sum_{(i,j)\in\mathcal{E}} q_{ij} y_i y_j
\]
be a quadratic real-valued function over binary variables
$\mathbf{y}\in\{0,1\}^n$. If the readout contains the binary variables
$\mathbf{y}$, or an affine representation from which they can be recovered,
then a one-hidden-layer ReLU energy head with at most
$n+|\mathcal{E}|$ hidden units can represent $Q(\mathbf{y})$ exactly on the
binary domain. If the architecture includes skip connections from the readout
input to the output layer, then $|\mathcal{E}|$ hidden units are sufficient.
Moreover, no factorized linear score can represent a quadratic function with a
nonzero interaction coefficient.
\end{proposition}

\begin{proof}
For binary variables,
\[
y_i y_j = \max(0, y_i+y_j-1).
\]
Thus each quadratic interaction can be represented by one ReLU unit with
pre-activation $y_i+y_j-1$. The linear terms can be represented using
additional ReLU units since $y_i=\max(0,y_i)$ for $y_i\in\{0,1\}$. The final
linear layer combines these units with weights $q_{ij}$ and $a_i$, and adds the
bias $c$. Hence a one-hidden-layer ReLU network with at most
$n+|\mathcal{E}|$ hidden units represents $Q$ exactly on the binary domain.

To show that factorized linear scoring cannot represent nonzero interactions,
fix all coordinates of $\mathbf{y}$ except $y_i$ and $y_j$. For any factorized
linear score
\[
L(\mathbf{y}) = b + \sum_{\ell=1}^n w_\ell y_\ell,
\]
the mixed finite difference with respect to coordinates $i$ and $j$ is
\[
\Delta_{ij} L
=
L(\mathbf{y}^{11})
-
L(\mathbf{y}^{10})
-
L(\mathbf{y}^{01})
+
L(\mathbf{y}^{00})
=0,
\]
where $\mathbf{y}^{ab}$ denotes the vector obtained by setting
$y_i=a$ and $y_j=b$ while keeping all other coordinates fixed. For the
quadratic interaction $q_{ij}y_i y_j$, the same finite difference equals
$q_{ij}$. Therefore, if $q_{ij}\neq 0$, no factorized linear score can
represent the function exactly.
\end{proof}

Proposition~\ref{prop:quadratic_expressivity} shows that the nonlinear energy
head is not merely an implementation detail. It strictly enlarges the class of
binary objectives representable by exact Neural-MIP inference beyond additive
per-variable scoring.

\section{Benchmark Problem Formulations}
\label{app:benchmarks}

We provide the complete MIP formulations for all three benchmarks.
Instance generation procedures and dataset statistics are described
in Appendix~\ref{app:data}.

\subsection{Combinatorial Auctions (CA)}
\label{app:formulations:ca}

The combinatorial auction problem~\citep{leyton2000towards}
allocates a set of items $\mathcal{C}$ to a set of bids $\mathcal{B}$
so as to maximize total revenue. Each bid $b \in \mathcal{B}$ specifies
a bundle $\mathcal{N}(b) \subseteq \mathcal{C}$ together with a value
$p_b > 0$. Let $y_b \in \{0,1\}$ indicate whether bid $b$ is accepted.
The resulting 0--1 packing formulation is
\begin{equation}
    \max_{\mathbf{y} \in \{0,1\}^{|\mathcal{B}|}} \;\;
    \sum_{b \in \mathcal{B}} p_b \, y_b
    \quad \text{s.t.} \quad
    \sum_{b:\, c \in \mathcal{N}(b)} y_b \leq 1
    \quad \forall c \in \mathcal{C}.
    \label{eq:ca_milp}
\end{equation}
The packing constraints enforce that each item is allocated to at most
one accepted bid. The problem is NP-hard in general~\citep{leyton2000towards}.

We generate CA instances using the arbitrary-relations generator of
\citet{leyton2000towards}. We consider a custom CA training distribution with
$(|\mathcal{C}|, |\mathcal{B}|) = (100, 175)$ and evaluate cross-scale
generalization on a harder test distribution with
$(|\mathcal{C}|, |\mathcal{B}|) = (100, 250)$. Exact reference
solutions are generated with a time limit of $3600$\,s per instance.

\paragraph{Graph encoding.}
We represent each instance as a bipartite graph with item nodes
$\mathcal{C}$ and bid nodes $\mathcal{B}$. An edge $(c,b)$ is present if
and only if item $c$ belongs to bid $b$, i.e., if $A_{cb}=1$ in the
incidence matrix $\mathbf{A} \in \{0,1\}^{|\mathcal{C}| \times |\mathcal{B}|}$.
Our encoder is therefore a bipartite message-passing neural network over
constraint-side item nodes and variable-side bid nodes. 

\paragraph{Features.}
For CA, item-node features encode local participation statistics
(e.g., item degree and dummy-item indicators), while bid-node features
encode bid value, bundle size, value-per-item, and a structural conflict
proxy derived from competing bids. Both node types additionally receive a
small amount of repeated global context, such as instance density and
average bid statistics. Edge features encode normalized incidence-level
signals linking item popularity, bid value, and bundle size.

\paragraph{Readout.}
Let $\mathbf{h}_b$ and $\mathbf{h}_c$ denote the final embeddings of bid
$b$ and item $c$, respectively. For each item, an auxiliary continuous variable
$u_c$ indicates whether it is occupied by at least one selected bid:
\begin{equation}
u_c
=
\mathbb{I}
\left[
\sum_{b:\,c\in\mathcal{N}(b)}y_b \geq 1
\right].
\label{eq:ch5_ca_item_occupancy}
\end{equation}
This relationship is represented exactly using
\begin{align}
u_c
&\leq
\sum_{b:\,c\in\mathcal{N}(b)}y_b,
&& \forall c\in\mathcal{C},
\label{eq:ch5_ca_occupancy_upper}\\
u_c
&\geq y_b,
&& \forall c\in\mathcal{C},
\quad
\forall b:\,c\in\mathcal{N}(b).
\label{eq:ch5_ca_occupancy_lower}
\end{align}
The corresponding occupied-item embedding pool is
\begin{equation}
\mathbf{h}_{\mathrm{occ}}(\mathbf{y})
=
\sum_{c\in\mathcal{C}}u_c\mathbf{h}_c.
\label{eq:ch5_ca_occupied_pool}
\end{equation}

The readout additionally includes the compact scalar-summary vector
\begin{equation}
\mathbf{s}_{\mathrm{CA}}(\mathbf{y})
=
\left[
s_{\mathrm{bid}},
s_{\mathrm{occ}},
s_{\mathrm{dummy}},
s_{\mathrm{conf}},
s_{\mathrm{mass}},
s_{\mathrm{value}},
s_{\mathrm{bundle}},
s_{\mathrm{risk}}
\right],
\label{eq:ch5_ca_scalar_summary}
\end{equation}
whose components are the selected-bid fraction, occupied-item fraction,
occupied dummy-item fraction, conflicted-item fraction, normalized conflict
mass, normalized selected-value sum, normalized selected-bundle-size sum,
and normalized selected conflict-risk sum, respectively.

The complete CA readout is
\begin{equation}
\mathbf{z}^{0}_{\mathrm{CA}}(\mathbf{y})
=
\left[
\sum_{b\in\mathcal{B}}y_b\mathbf{h}_b,
\;
\sum_{c\in\mathcal{C}}u_c\mathbf{h}_c,
\;
\mathbf{s}_{\mathrm{CA}}(\mathbf{y}),
\;
\mathbf{h}_{\mathcal{G}}
\right].
\label{eq:ch5_ca_readout}
\end{equation}
The first pool represents the accepted bids, while the second captures their
consequences on the item side of the bipartite graph. The scalar summaries
provide compact information about selection density, item occupancy,
aggregate value, bundle size, and conflict structure.

\subsection{Quadratic Knapsack (QKP)}
\label{app:formulations:qkp}

The Quadratic Knapsack Problem~\citep{pisinger2007quadratic} selects
a subset of $n$ items to maximize a quadratic objective subject to a
linear capacity constraint:
\begin{equation}
    \max_{\mathbf{y} \in \{0,1\}^n} \;\;
    \sum_{i=1}^n p_i \, y_i \;+\;
    \sum_{i=1}^n \sum_{j=i+1}^n Q_{ij} \, y_i y_j
    \quad \text{s.t.} \quad
    \sum_{i=1}^n w_i \, y_i \leq W,
    \label{eq:qkp_milp}
\end{equation}
where $p_i$ is the linear profit of item $i$, $Q_{ij}$ is the pairwise
interaction profit between items $i$ and $j$, $w_i$ is the weight of
item $i$, and $W$ is the knapsack capacity. The quadratic term captures
both positive and negative interactions between co-selected items, which
motivates learning an energy function over complete configurations rather
than independent per-item scores.

We consider instances with $n=150$ items and interaction density
$\rho=0.30$. Linear profits are sampled as
$p_i \sim \mathcal{U}[0,10]$, item weights as
$w_i \sim \mathcal{U}[1,10]$, and the knapsack capacity is set to
$W = 0.5 \sum_i w_i$. Pairwise interactions are obtained from a
symmetric random matrix with off-diagonal entries drawn from
$\mathcal{U}[-2,8]$, which is then sparsified to the target density
$\rho$. Reference solutions are computed with Gurobi~\citep{gurobi}
using a time limit of $3600$\,s per instance.

\paragraph{Graph encoding.}
We represent each instance as a bipartite graph with item nodes
$I$ and interaction nodes $P$.
Each interaction node $p=(i,j)\in P$ corresponds to a nonzero
upper-triangular pairwise coefficient $Q_{ij}$ with $i<j$,
and is connected to both item nodes $i$ and $j$.
This construction allows message passing to propagate pairwise
interaction information into the item embeddings used by the
downstream readout.

\paragraph{Features.}
For QKP, item-node features encode linear profit, weight, interaction
degree, and average interaction strength, together with repeated global
context such as graph density and aggregate interaction scale.
Interaction-node features encode the magnitude and sign of the pairwise
term, along with simple summaries of the associated item weights and
profits. Edge features transmit normalized interaction strength and
capacity-relevant pair information.

\paragraph{Readout.}
Each QKP instance is represented as a bipartite graph between item nodes
$\mathcal{I}$ and interaction nodes $\mathcal{P}$. Each interaction node
$p=(i,j)\in\mathcal{P}$ corresponds to a nonzero upper-triangular pairwise
coefficient $Q_{ij}$ and is connected to item nodes $i$ and $j$. The
resulting item embeddings $\mathbf{h}_i$ therefore encode both item-level
attributes and information propagated from the quadratic interactions.

Let $y_i\in\{0,1\}$ indicate whether item $i$ is selected. In addition to
the selected-item embedding pool, QKP uses the unselected-item pool
\begin{equation}
\label{eq:qkp_unselected_pool}
\mathbf{h}_{\mathrm{unsel}}(\mathbf{y})
=
\sum_{i\in\mathcal{I}}(1-y_i)\mathbf{h}_i,
\end{equation}
the selected-item fraction
\begin{equation}
\label{eq:qkp_selected_fraction}
r_{\mathrm{sel}}(\mathbf{y})
=
\frac{1}{|\mathcal{I}|}
\sum_{i\in\mathcal{I}}y_i,
\end{equation}
and the normalized capacity usage
\begin{equation}
\label{eq:qkp_capacity_fraction}
r_{\mathrm{cap}}(\mathbf{y})
=
\frac{1}{W}
\sum_{i\in\mathcal{I}}w_i y_i.
\end{equation}

The additional structured component is
\begin{equation}
\label{eq:qkp_zlin}
\mathbf{z}_{\mathrm{lin}}^{\mathrm{QKP}}(\mathbf{y})
=
\left[
\mathbf{h}_{\mathrm{unsel}}(\mathbf{y}),
\;
r_{\mathrm{sel}}(\mathbf{y}),
\;
r_{\mathrm{cap}}(\mathbf{y})
\right].
\end{equation}
The complete QKP readout is
\begin{equation}
\label{eq:qkp_complete_readout}
\mathbf{z}_{\mathrm{QKP}}^{0}(\mathbf{y})
=
\left[
\sum_{i\in\mathcal{I}}y_i\mathbf{h}_i,
\;
\sum_{i\in\mathcal{I}}(1-y_i)\mathbf{h}_i,
\;
\frac{1}{|\mathcal{I}|}\sum_{i\in\mathcal{I}}y_i,
\;
\frac{1}{W}\sum_{i\in\mathcal{I}}w_i y_i,
\;
\mathbf{h}_{\mathcal{G}}
\right].
\end{equation}

The selected- and unselected-item pools expose information about both sides
of the current partition of the item set. The two scalar summaries describe
the size of the selected set and its consumption of the available knapsack
capacity. All four decision-dependent components are affine in
$\mathbf{y}$.

\subsection{Personalized Medicine Supply Chain (CART)}
\label{app:formulations:cart}

The personalized medicine supply chain problem optimizes the strategic 
selection of cell therapy manufacturing facilities and the routing of 
patient therapies through a multi-echelon logistics network under 
strict capacity and turnaround-time 
constraints~\citep{triantafyllou2022digital}.

\paragraph{Problem structure.}
The problem decomposes naturally into two levels:
\begin{itemize}[leftmargin=1.5em]
    \item \textbf{Strategic master problem}: select which subset of 
    $|\mathcal{M}|$ candidate manufacturing facilities to open. This 
    involves $|\mathcal{M}|$ binary variables $y_m \in \{0,1\}$.
    \item \textbf{Operational subproblem}: given open facilities, 
    route all patient therapies through the network to minimize total 
    cost subject to capacity and timing constraints. This involves 
    continuous and binary flow variables but becomes tractable once 
    the facility set is fixed.
\end{itemize}
The monolithic MILP is intractable at scale because the number of 
patient routing variables grows as $|P| \times |C| \times |M| \times 
|J| \times |T|$, which reaches millions for realistic instances. 
EnerGNN learns an energy over facility location configurations $\mathbf{y} = 
\{y_m\}$ which is then optimized under the strategic feasibility constraints; the operational subproblem induced by the selected configuration is subsequently solved exactly by Gurobi.

\paragraph{Full MILP formulation.}
We provide the detailed mixed integer linear programming (MILP) formulation for the personalized medicine CAR T-cell therapy supply chain problem. The model jointly optimizes facility location, patient assignment, and time-dependent logistics under capacity and turnaround-time constraints.

\subparagraph{Sets and indices}
Let \(\mathcal{P}\), \(\mathcal{C}\), \(\mathcal{M}\), \(\mathcal{H}\), \(\mathcal{J}\), and \(\mathcal{T}\) denote the sets of patients, clinical leukapheresis sites, candidate manufacturing facilities, hospitals, transportation modes, and discrete time periods, respectively. Indices are denoted by \(p\), \(c\), \(m\), \(h\), \(j\), and \(t\).

\subparagraph{Parameters}
\begin{itemize}
    \item \(C^{\mathrm{inv}}_m\): capital investment cost for opening manufacturing facility \(m\).
    \item \(C^{\mathrm{fix}}_m\): fixed operating cost of facility \(m\).
    \item \(C^{\mathrm{var}}_p\): patient-specific variable manufacturing cost.
    \item \(C^{\mathrm{qc}}\): quality-control cost per therapy.
    \item \(K_m\): maximum processing capacity, or number of parallel manufacturing lines, at facility \(m\).
    \item \(\delta_{pct}\): binary parameter equal to 1 if patient \(p\) arrives at clinical site \(c\) at time \(t\).
    \item \(\tau^{\mathrm{lp}}\): duration of leukapheresis.
    \item \(\tau^{\mathrm{man}}\): duration of manufacturing.
    \item \(\tau^{\mathrm{qc}}\): duration of quality control.
    \item \(\tau^{\mathrm{out}}_j\): transportation time from a clinical site to a manufacturing facility using mode \(j\).
    \item \(\tau^{\mathrm{ret}}_j\): transportation time from a manufacturing facility to a hospital using mode \(j\).
    \item \(C^{\mathrm{out}}_{cmj}\): transportation cost from clinical site \(c\) to facility \(m\) using mode \(j\).
    \item \(C^{\mathrm{ret}}_{mhj}\): transportation cost from facility \(m\) to hospital \(h\) using mode \(j\).
    \item \(B\): maximum number of manufacturing facilities that can be opened.
    \item \(D\): maximum allowable vein-to-vein turnaround time.
    \item \(N_P\): total number of patients in the planning horizon.
\end{itemize}

Variables with time indices outside \(\mathcal{T}\) are fixed to zero.

\subparagraph{Decision variables}

\textit{Facility location variables.}
\[
y_m \in \{0,1\}, \qquad \forall m\in\mathcal{M},
\]
where \(y_m=1\) if manufacturing facility \(m\) is opened.

\textit{Network design variables.}
\[
u_{cm} \in \{0,1\}, \qquad \forall c\in\mathcal{C},\,m\in\mathcal{M},
\]
where \(u_{cm}=1\) if clinical site \(c\) is connected to manufacturing facility \(m\).

\[
v_{mh} \in \{0,1\}, \qquad \forall m\in\mathcal{M},\,h\in\mathcal{H},
\]
where \(v_{mh}=1\) if manufacturing facility \(m\) is connected to hospital \(h\).

\textit{Transportation variables.}
\[
x_{pcmjt} \in \{0,1\},
\qquad
\forall p\in\mathcal{P},\,c\in\mathcal{C},\,m\in\mathcal{M},\,j\in\mathcal{J},\,t\in\mathcal{T},
\]
where \(x_{pcmjt}=1\) if patient \(p\)'s material is transported from clinical site \(c\) to facility \(m\) using mode \(j\) at time \(t\).

\[
w_{pmhjt} \in \{0,1\},
\qquad
\forall p\in\mathcal{P},\,m\in\mathcal{M},\,h\in\mathcal{H},\,j\in\mathcal{J},\,t\in\mathcal{T},
\]
where \(w_{pmhjt}=1\) if therapy \(p\) is transported from facility \(m\) to hospital \(h\) using mode \(j\) at time \(t\).

\textit{Flow variables.}
\[
\ell^{\mathrm{dep}}_{pcmjt} \ge 0,
\qquad
\ell^{\mathrm{arr}}_{pcmjt} \ge 0,
\]
denote, respectively, material leaving clinical site \(c\) for facility \(m\), and material arriving at facility \(m\), for patient \(p\), mode \(j\), and time \(t\).

\[
q^{\mathrm{in}}_{pmt} \ge 0,
\qquad
q^{\mathrm{out}}_{pmt} \ge 0,
\]
denote, respectively, therapy \(p\) entering and leaving manufacturing facility \(m\) at time \(t\).

\[
g^{\mathrm{out}}_{pct} \ge 0,
\qquad
g^{\mathrm{in}}_{pht} \ge 0,
\]
denote, respectively, material leaving clinical site \(c\), and final therapy arriving at hospital \(h\), for patient \(p\) at time \(t\).

\textit{Timing variables.}
\[
s_p, e_p, r_p \ge 0,
\qquad \forall p\in\mathcal{P},
\]
denote the start time, completion time, and turnaround time of patient \(p\).

\subparagraph{Objective function}

The objective minimizes total supply chain cost, including facility investment, fixed operating cost, patient-specific manufacturing cost, quality control, and transportation:

\begin{equation}
\begin{aligned}
\min \quad Z
=
&\sum_{m\in\mathcal{M}}
\left(C^{\mathrm{inv}}_m + C^{\mathrm{fix}}_m\right)y_m
+
\sum_{p\in\mathcal{P}}
\left(C^{\mathrm{var}}_p + C^{\mathrm{qc}}\right) \\
&+
\sum_{p\in\mathcal{P}}
\sum_{c\in\mathcal{C}}
\sum_{m\in\mathcal{M}}
\sum_{j\in\mathcal{J}}
\sum_{t\in\mathcal{T}}
C^{\mathrm{out}}_{cmj} x_{pcmjt} \\
&+
\sum_{p\in\mathcal{P}}
\sum_{m\in\mathcal{M}}
\sum_{h\in\mathcal{H}}
\sum_{j\in\mathcal{J}}
\sum_{t\in\mathcal{T}}
C^{\mathrm{ret}}_{mhj} w_{pmhjt}.
\end{aligned}
\label{eq:objective}
\end{equation}

Equivalently, if facility investment and fixed operating costs are allocated per patient as in the original implementation, the patient-level manufacturing cost may be written as
\begin{equation}
C^{\mathrm{man}}_p
=
\frac{
\sum_{m\in\mathcal{M}}
\left(C^{\mathrm{inv}}_m + C^{\mathrm{fix}}_m\right)y_m
}{N_P}
+
C^{\mathrm{var}}_p,
\qquad \forall p\in\mathcal{P}.
\label{eq:patient-manufacturing-cost}
\end{equation}

\subparagraph{Constraints}

\textit{Leukapheresis completion.}
A patient sample can leave the clinical site only after the leukapheresis duration has elapsed:
\begin{equation}
g^{\mathrm{out}}_{pc,t+\tau^{\mathrm{lp}}}
=
\delta_{pct},
\qquad
\forall p\in\mathcal{P},\,c\in\mathcal{C},\,t\in\mathcal{T}.
\label{eq:leukapheresis}
\end{equation}

\textit{Outbound shipment selection.}
Each patient is assigned to exactly one outbound shipment:
\begin{equation}
\sum_{c\in\mathcal{C}}
\sum_{m\in\mathcal{M}}
\sum_{j\in\mathcal{J}}
\sum_{t\in\mathcal{T}}
x_{pcmjt}
=
1,
\qquad
\forall p\in\mathcal{P}.
\label{eq:outbound-assignment}
\end{equation}

The outbound shipment variables define the material leaving the clinical site:
\begin{equation}
\ell^{\mathrm{dep}}_{pcmjt}
=
x_{pcmjt},
\qquad
\forall p,c,m,j,t.
\label{eq:outbound-flow-binary-link}
\end{equation}

All material leaving a clinical site must be assigned to a manufacturing facility and a transportation mode:
\begin{equation}
g^{\mathrm{out}}_{pct}
=
\sum_{m\in\mathcal{M}}
\sum_{j\in\mathcal{J}}
\ell^{\mathrm{dep}}_{pcmjt},
\qquad
\forall p,c,t.
\label{eq:clinical-departure-balance}
\end{equation}

\textit{Outbound transportation timing.}
Material arriving at the manufacturing facility is delayed by the transportation time of the selected mode:
\begin{equation}
\ell^{\mathrm{arr}}_{pcmj,t+\tau^{\mathrm{out}}_j}
=
\ell^{\mathrm{dep}}_{pcmjt},
\qquad
\forall p,c,m,j,t.
\label{eq:outbound-delay}
\end{equation}

The total material entering a manufacturing facility equals all arrivals from clinical sites:
\begin{equation}
q^{\mathrm{in}}_{pmt}
=
\sum_{c\in\mathcal{C}}
\sum_{j\in\mathcal{J}}
\ell^{\mathrm{arr}}_{pcmjt},
\qquad
\forall p,m,t.
\label{eq:manufacturing-arrival}
\end{equation}

\textit{Manufacturing and quality control timing.}
A therapy can leave the manufacturing facility only after manufacturing and quality control are completed:
\begin{equation}
q^{\mathrm{out}}_{pm,t+\tau^{\mathrm{man}}+\tau^{\mathrm{qc}}}
=
q^{\mathrm{in}}_{pmt},
\qquad
\forall p,m,t.
\label{eq:manufacturing-qc-delay}
\end{equation}

\textit{Return shipment after manufacturing.}
Every completed therapy must be assigned to exactly one return shipment:
\begin{equation}
\sum_{h\in\mathcal{H}}
\sum_{j\in\mathcal{J}}
w_{pmhjt}
=
q^{\mathrm{out}}_{pmt},
\qquad
\forall p,m,t.
\label{eq:return-shipment-selection}
\end{equation}

\textit{Hospital arrival timing.}
The final therapy arrives at the hospital after the return transportation delay:
\begin{equation}
g^{\mathrm{in}}_{ph,t+\tau^{\mathrm{ret}}_j}
=
\sum_{m\in\mathcal{M}}
w_{pmhjt},
\qquad
\forall p,h,j,t.
\label{eq:hospital-arrival-delay}
\end{equation}

Each patient receives exactly one final therapy delivery:
\begin{equation}
\sum_{h\in\mathcal{H}}
\sum_{t\in\mathcal{T}}
g^{\mathrm{in}}_{pht}
=
1,
\qquad
\forall p\in\mathcal{P}.
\label{eq:final-delivery}
\end{equation}

\textit{Manufacturing capacity.}
The number of therapies simultaneously occupying facility \(m\) cannot exceed the installed capacity:
\begin{equation}
\sum_{p\in\mathcal{P}}
\left(
\sum_{\tau=0}^{t} q^{\mathrm{in}}_{pm\tau}
-
\sum_{\tau=0}^{t} q^{\mathrm{out}}_{pm\tau}
\right)
\le
K_m y_m,
\qquad
\forall m\in\mathcal{M},\,t\in\mathcal{T}.
\label{eq:capacity}
\end{equation}

\textit{Network design and facility activation.}
A clinical site can only be connected to an opened manufacturing facility:
\begin{equation}
u_{cm} \le y_m,
\qquad
\forall c\in\mathcal{C},\,m\in\mathcal{M}.
\label{eq:clinical-facility-link}
\end{equation}

A manufacturing facility can only be connected to a hospital if the facility is opened:
\begin{equation}
v_{mh} \le y_m,
\qquad
\forall m\in\mathcal{M},\,h\in\mathcal{H}.
\label{eq:facility-hospital-link}
\end{equation}

Outbound shipments are allowed only along selected clinical-site-to-facility links:
\begin{equation}
x_{pcmjt} \le u_{cm},
\qquad
\forall p,c,m,j,t.
\label{eq:outbound-link-activation}
\end{equation}

Return shipments are allowed only along selected facility-to-hospital links:
\begin{equation}
w_{pmhjt} \le v_{mh},
\qquad
\forall p,m,h,j,t.
\label{eq:return-link-activation}
\end{equation}

The total number of established manufacturing facilities is bounded:
\begin{equation}
\sum_{m\in\mathcal{M}} y_m \le B.
\label{eq:facility-budget}
\end{equation}

\textit{Start, completion, and turnaround times.}
The start time of each patient is determined by the clinical site arrival time:
\begin{equation}
s_p
=
\sum_{c\in\mathcal{C}}
\sum_{t\in\mathcal{T}}
t\,\delta_{pct},
\qquad
\forall p\in\mathcal{P}.
\label{eq:start-time}
\end{equation}

The completion time is the arrival time of the final therapy at the hospital:
\begin{equation}
e_p
=
\sum_{h\in\mathcal{H}}
\sum_{t\in\mathcal{T}}
t\,g^{\mathrm{in}}_{pht},
\qquad
\forall p\in\mathcal{P}.
\label{eq:completion-time}
\end{equation}

Turnaround time is defined as the time between patient arrival and therapy delivery:
\begin{equation}
r_p = e_p - s_p,
\qquad
\forall p\in\mathcal{P}.
\label{eq:turnaround-definition}
\end{equation}

The vein-to-vein turnaround time must not exceed the clinical deadline:
\begin{equation}
r_p \le D,
\qquad
\forall p\in\mathcal{P}.
\label{eq:turnaround-limit}
\end{equation}

\paragraph{Graph encoding.}
We represent each instance as a complete bipartite graph between
clinical facility-demand nodes $\mathcal{C}$ and candidate manufacturing
facility nodes $\mathcal{M}$. Demand-node features summarize the local
patient arrival profile at each clinical site using aggregate demand
statistics, including total demand, variability, peak
demand, empirical quantiles, and seasonal Fourier components. Facility
node features encode normalized capacity, operating cost rate, and
demand-weighted transport statistics. Each edge $(c,m)$ encodes the
transport cost between clinical site $c$ and facility $m$ using a
normalized scalar together with an RBF expansion. This compressed graph
representation exposes the aggregate operational structure of the full
supply chain without explicitly instantiating patient-level routing
variables in the learned model.

\paragraph{Features.}
In the implementation used for our experiments, demand-node features
have dimension 14 and include local demand statistics together with a
small amount of repeated global context. Facility node features have
dimension 9 and encode capacity, normalized fixed cost rate, and
transport summaries, again augmented with repeated global context. Edge
features have dimension 6, consisting of one normalized transport cost
feature and a five-dimensional RBF expansion.

\paragraph{Readout.}
Each CART instance is represented as a complete bipartite graph
\[
\mathcal{G}_{\mathrm{CART}}
=
(\mathcal{C}\cup\mathcal{M},\mathcal{E}),
\qquad
\mathcal{E}=\mathcal{C}\times\mathcal{M},
\]
where $\mathcal{C}$ denotes clinical facility-demand nodes and $\mathcal{M}$ denotes
candidate manufacturing-facility nodes. The encoder returns demand-node
embeddings $\mathbf{h}_c$, facility-node embeddings $\mathbf{h}_m$, and a
global graph embedding $\mathbf{h}_{\mathcal{G}}$.

Let $y_m\in\{0,1\}$ indicate whether manufacturing facility $m$ is selected.
For each demand node $c$, an auxiliary binary variable $q_c$ indicates
whether it is covered by at least one selected facility:
\begin{equation}
\label{eq:cart_coverage_definition}
q_c
=
\mathbb{I}
\left[
\sum_{m\in\mathcal{M}(c)}y_m \geq 1
\right],
\end{equation}
where $\mathcal{M}(c)$ is the set of candidate facilities connected to
demand node $c$. This logical relationship is represented exactly through
\begin{align}
q_c
&\leq
\sum_{m\in\mathcal{M}(c)}y_m,
&&\forall c\in\mathcal{C},
\label{eq:cart_coverage_upper}\\
q_c
&\geq y_m,
&&\forall c\in\mathcal{C},
\quad
\forall m\in\mathcal{M}(c).
\label{eq:cart_coverage_lower}
\end{align}

The additional decision-dependent component is the pooled embedding of the
covered demand nodes:
\begin{equation}
\label{eq:cart_zlin}
\mathbf{z}_{\mathrm{lin}}^{\mathrm{CART}}(\mathbf{y})
=
\sum_{c\in\mathcal{C}}q_c\mathbf{h}_c.
\end{equation}
The complete CART readout is therefore
\begin{equation}
\label{eq:cart_complete_readout}
\mathbf{z}_{\mathrm{CART}}^{0}(\mathbf{y})
=
\left[
\sum_{m\in\mathcal{M}}y_m\mathbf{h}_m,
\;
\sum_{c\in\mathcal{C}}q_c\mathbf{h}_c,
\;
\mathbf{h}_{\mathcal{G}}
\right].
\end{equation}

The first term represents the selected manufacturing facilities, while the
second represents the clinical facility-demand nodes covered by the selected facility
set. CART therefore uses pooled learned embeddings from both sides of the
bipartite graph and does not include additional manually constructed scalar
summaries.

\section{Neural-MIP Encoding}
\label{app:mip_encoding}

We encode the learned ReLU energy head exactly inside the MIP using a
standard Big-M formulation for ReLU activations. The benchmark-specific
combinatorial constraints and the readout-level auxiliary variables are
then combined with this head encoding to yield the final inference
model.

\subsection{ReLU Head Encoding}

Let
\[
a_j^l = \mathbf{w}_j^\top \mathbf{h}^{l-1} + b_j
\]
denote the affine pre-activation of neuron \(j\) in layer \(l\), and let
\[
h_j^l = \max(0,a_j^l)
\]
be the corresponding ReLU activation. Suppose valid finite bounds
\[
\underline{a}_j^l \le a_j^l \le \overline{a}_j^l
\]
are available. For neurons whose bounds satisfy
\(\underline{a}_j^l < 0 < \overline{a}_j^l\), we introduce a binary
activation indicator \(z_j^l\in\{0,1\}\) and impose the exact mixed integer
linear formulation
\begin{align}
    h_j^l &\ge a_j^l, \\
    h_j^l &\ge 0, \\
    h_j^l &\le a_j^l-\underline{a}_j^l(1-z_j^l), \\
    h_j^l &\le \overline{a}_j^l z_j^l.
\end{align}
If \(\underline{a}_j^l \ge 0\), the neuron is always active and we set
\(h_j^l=a_j^l\). If \(\overline{a}_j^l \le 0\), the neuron is always
inactive and we set \(h_j^l=0\). The final linear layer is encoded by
linear equality constraints, producing the scalar energy minimized by
the inference MIP.

\subsection{Interval Bound Propagation}
\label{app:ibp}

The practical strength of the Neural-MIP depends on the tightness of the
pre-activation bounds. We obtain these bounds using interval bound propagation
through the assembled readout and energy head.

For an affine layer
\[
\mathbf{a}^l = W^l \mathbf{h}^{l-1} + \mathbf{b}^l
\]
with elementwise input bounds
\([\underline{\mathbf{h}}^{\,l-1}, \overline{\mathbf{h}}^{\,l-1}]\),
the pre-activation bounds are
\begin{align}
    \underline{a}_j^l
    &=
    b_j^l + \sum_i
    \min\!\left(
        W_{ji}^l \underline{h}_i^{\,l-1},
        W_{ji}^l \overline{h}_i^{\,l-1}
    \right), \\
    \overline{a}_j^l
    &=
    b_j^l + \sum_i
    \max\!\left(
        W_{ji}^l \underline{h}_i^{\,l-1},
        W_{ji}^l \overline{h}_i^{\,l-1}
    \right).
\end{align}

The corresponding post-ReLU bounds are
\[
\underline{h}_j^l = \max(0,\underline{a}_j^l),
\qquad
\overline{h}_j^l = \max(0,\overline{a}_j^l).
\]
These pre-activation bounds are used directly in the exact ReLU formulation
above. When the bounds imply a fixed activation state, the corresponding
binary variable is removed. Overall, the benchmark-specific readout linearizations and subsequent interval
propagation provide the neuron-wise pre-activation bounds used in exact MIP
inference.

\subsection{Benchmark-Specific Input Bounds}
\label{app:input_bounds}
The input intervals for this propagation step are derived from the exact
linearized readout features used by each benchmark.

\paragraph{CART.}
For CART, the readout combines pooled selected facility embeddings,
pooled covered-demand embeddings, and a global graph embedding. Bounds
for the selected-facility pooling terms are obtained from the facility
cardinality constraints \(k_{\min} \le \|\mathbf{y}\|_1 \le k_{\max}\).
Coverage terms are linearized using auxiliary binary variables
indicating whether each demand node is served by at least one selected
facility, and their bounds are computed from sign-split sums over the
corresponding coefficients.

\paragraph{CA.}
For CA, the readout includes pooled selected bid and occupied item
embeddings together with scalar summaries such as selected bid fraction,
occupied item fraction, and aggregate selected bid statistics. These are encoded with auxiliary continuous occupancy variables and count-based one-hot linearizations, which allow exact bounds to be propagated through the resulting MLP input.

\paragraph{QKP.}
For QKP, the readout pools selected item embeddings and, in the
size-aware variant, also includes unselected item summaries, selected
fraction, and used capacity mass. Under the knapsack constraint,
selected pool bounds are tightened using a fractional knapsack
relaxation rather than naive unconstrained cardinality intervals. This
yields substantially tighter input bounds for the downstream ReLU
encoding.

Overall, the benchmark-specific readout linearizations and the
subsequent interval propagation provide the neuron-wise Big-M constants
used in exact MIP inference.

\subsection{Problem-Specific Hard Feasibility Constraints}
\label{app:hard_constraints}
Beyond the energy encoding, the following hard constraints are added 
to the inference MIP for each benchmark:

\paragraph{CART.}
\begin{align}
    \sum_m y_m &\leq k_{\max}, \label{eq:cart-cardinality}\\
    \sum_m \kappa_m y_m &\geq \kappa_{\mathrm{ref}}, \label{eq:cart-capacity}\\
    y_j &\leq y_i \qquad \forall (i,j) \in \mathcal{S}. \label{eq:cart-symmetry}
\end{align}
where \eqref{eq:cart-cardinality} opens at most $k_{\max}$ facilities,
\eqref{eq:cart-capacity} requires total open capacity to meet the peak
reference demand $\kappa_{\mathrm{ref}}$, and \eqref{eq:cart-symmetry}
imposes symmetry-breaking cuts over pairs of identical capacity facilities
$\mathcal{S}$. 

\paragraph{CA.}
\begin{align}
    \sum_{b: c \in \mathcal{N}(b)} y_b \leq 1
    \qquad \forall c \in \mathcal{C}.
    \label{eq:ca-packing}
\end{align}
where \eqref{eq:ca-packing} enforces that each item is allocated to at most one bid.
Zero-shot injection experiments additionally introduce cardinality 
bounds $\sum_b y_b \leq K$ and bidder exclusion constraints 
$y_i + y_j \leq 1$ for specified pairs $(i,j)$.

\paragraph{QKP.}
\begin{align}
    \sum_i w_i y_i \leq W
    \label{eq:qkp-knapsack}
\end{align}
where \eqref{eq:qkp-knapsack} enforces that the total selected weight remains within the budget $W$.

\section{Dataset Splits and Candidate Pool Construction}
\label{app:data}

\paragraph{Dataset splits.}
For each benchmark, we use 800 training instances, 100 validation instances,
and 100 held-out test instances. 

For each training instance, we construct a candidate pool
$\mathcal{Y} = \mathcal{F} \cup \mathcal{I}$ containing feasible and
infeasible configurations. Pool construction is benchmark-specific.

\noindent\textbf{CART.}
For CART, each instance has 16 candidate facilities. We first solve the
unrestricted patient-level MIP to global optimality and collect incumbent
facility selection configurations from branch-and-bound. Additional candidates
are generated using capacity- and transport-guided heuristics, local Hamming
perturbations, swaps, facility additions and removals, and random configurations
with varying cardinality. Each candidate is evaluated by solving the corresponding
tactical patient scheduling MIP. Feasible configurations attaining the proven
global optimum within numerical tolerance are labeled positive, other feasible
configurations are labeled feasible negatives, and configurations certified
infeasible are labeled infeasible negatives. We retain 160 candidates per instance
through stratified selection, initially keeping up to three positives and targeting
30\% infeasible negatives among the remaining candidates.

\noindent\textbf{CA.}
For CA, we first solve each instance with Gurobi and collect incumbent
solutions from branch-and-bound. Feasible solutions with relative primal gap
at most $\epsilon_{\text{pos}} = 0.01$ are labelled positive, up to a
maximum of 50 positives per instance. We then construct 64 feasible
negatives using a quota-driven mixture of additional incumbent
solutions, guided local Hamming perturbations, and farther Hamming
perturbations of the best solution. The feasible-negative selection is
balanced across Hamming distance buckets and primal gap bands. Finally, we
generate 50 infeasible negatives by perturbing feasible masks into
different violation bands, corresponding to mild, moderate, and severe
item conflicts.

\noindent\textbf{QKP.}
For QKP, we impose a one-hour time limit on Gurobi per instance. Training and validation may include instances for which global optimality is not certified within this limit. In contrast, the test set contains instances
with Gurobi certified global optima, allowing all reported primal gaps
to be measured against proven optimal solutions.
For each training instance, we retain the best available solution
together with all incumbent solutions whose relative primal gap is
at most $\epsilon_{\text{pos}} = 0.01$; these form the positive subset.
We then generate 128 feasible negatives and 64 infeasible negatives
using stochastic mask proposals. Each proposal is either a random binary
mask with random density or a mutation of the best mask obtained by
flipping a random subset of items. Proposed masks that satisfy the
knapsack constraint are retained as feasible negatives, while those that
violate the capacity constraint are retained as infeasible negatives.

\subsection{Offline Computational Cost and Amortization}
\label{app:ch5_amortisation}

The main benchmark runtimes report the online inference and optimization
phase. EnerGNN also incurs a one-time offline cost for solver-based candidate
generation and neural model training. Table~\ref{tab:offline_costs} reports the
complete elapsed wall-clock cost of this offline phase. Candidate generation was distributed across 1000 submitted jobs,
with approximately 80 jobs running concurrently, subject to
scheduler availability. Reported elapsed wall-clock times include
queueing effects and therefore depend on cluster availability.

\begin{table}[H]
\caption{Offline computational cost and approximate amortization break-even
points.}
\centering
\small
\begin{tabular}{lrrrr}
\toprule
\textbf{Benchmark}
& \textbf{Label generation}
& \textbf{Training}
& \textbf{Total offline}
& \textbf{Break-even instances} \\
\midrule
QKP  & 14.7 h & 3.39 h & 18.1 h & $\approx125$ \\
CA   & 0.12 h & 5.54 h & 5.66 h & Not reached \\
CART & 123.2 h & 3.52 h & 126.72 h & $\approx130$ \\
\bottomrule
\end{tabular}
\label{tab:offline_costs}
\end{table}

For a benchmark where EnerGNN is faster online than solving the full problem,
the approximate number of deployment instances required to recover the
offline investment is calculated as
\[
N_{\mathrm{break-even}}
=
\frac{T_{\mathrm{offline}}}
     {T_{\mathrm{full}}-T_{\mathrm{EnerGNN}}},
\]
where $T_{\mathrm{offline}}$ includes label generation and training,
$T_{\mathrm{full}}$ is the mean time required to solve a new instance using
the original optimisation formulation, and $T_{\mathrm{EnerGNN}}$ is the
mean online EnerGNN time.

The offline investment is recovered after approximately 125 QKP instances
and 130 CART instances under the reported experimental conditions. No
break-even point is reported for CA because Gurobi is already faster on the
original auction formulation. These results make explicit that EnerGNN is
intended for repeated decision-making settings in which the one-time
solver supervision and training costs can be amortised across many deployment
instances.

\subsection{Computational Setup}
\label{app:computational_setup}

Experiments were conducted using the Imperial College London CX3
high-performance computing cluster and a local workstation.
Training, validation, and test datasets were generated on CPU
nodes with up to 250 GB of RAM. Neural models were trained
offline on GPU nodes equipped with NVIDIA Quadro RTX 6000 GPUs,
using one GPU, 4 CPU cores, and 24 GB of RAM.
Test time optimization and Neural-MIP inference were performed
on a workstation with 64 GB of RAM. Optimization runs used Gurobi 12.0.2.

Reported runtimes exclude dataset generation, candidate pool
construction, and neural network training. For learning-based
methods, reported online times include neural prediction or
encoding, construction of the corresponding optimization problem,
and all solver calls required to produce the final solution.

\section{Loss Configuration Details}
\label{app:loss_details}

\subsection{Loss Configuration Used in Reported Experiments}

All reported experiments use the same loss formulation, consisting of an
optimality-aware contrastive loss term, a feasible calibration term, and an infeasibility shaping term. The optimality-aware contrastive term is applied over the full
candidate pool. Across these runs, we use a bipartite MPNN encoder with hidden dimension
128, two message-passing rounds, and dropout 0.1. The reported runs use a shallow MLP energy head with one hidden layer of width 32.

\subsection{Training and Model Selection}

All models are trained using Adam with the benchmark-specific hyperparameters
reported in Table~\ref{tab:hyperparams}. Model selection is performed on the
validation set. For each run, we retain the checkpoint with the lowest
validation mean primal gap and report final test results from this checkpoint.

\paragraph{Adaptive Huber scale.}
For each candidate pool, the Huber transition scale is computed as
\begin{equation}
    \delta_{\mathcal{F}}
    =
    \min\!\left\{
    \delta_{\max},
    \;
    \max\!\left[
    \delta_{\min},
    \;
    \kappa_\delta
    \left(
    \max_{i\in\mathcal{F}} g_i
    -
    \min_{i\in\mathcal{F}} g_i
    \right)
    \right]
    \right\}.
\end{equation}

\begin{table}[h]
\setlength{\belowcaptionskip}{1pt}
\caption{Training hyperparameters for the full EnerGNN runs reported in the paper.
Here $\lambda_{\mathrm{inf}}$ denotes the infeasible shaping weight,
$T$ is the prediction softmax temperature, $\tau_{\mathrm{rank}}$ is the
target ranking temperature,
$\delta_{\mathrm{inf}}$ is the infeasible Huber transition scale.}
\centering
\small
\begin{tabular}{lccc}
\toprule
Hyperparameter & CART & CA & QKP \\
\midrule
Learning rate & $2\times 10^{-6}$ & $1\times 10^{-6}$ & $1\times 10^{-5}$ \\
Weight decay & $1\times 10^{-5}$ & $1\times 10^{-5}$ & $1\times 10^{-5}$ \\
Epochs & 2000 & 2000 & 2000 \\
$\lambda_{\mathrm{cal}}$ & 1.0 & 0.01 & 0.05 \\
$\lambda_{\mathrm{inf}}$ & 0.1 & 0.1 & 0.1 \\
$T$ & 0.1 & 0.2 & 0.5 \\
$\tau_{\mathrm{rank}}$ & 0.05 & 0.005 & 0.05 \\
$\delta_{\mathrm{inf}}$ & 0.1 & 0.1 & 0.1 \\
$\kappa_{\delta}$ & 0.15 & 0.15 & 0.15 \\
$\delta_{\min}$ & 0.01 & 0.01 & 0.01 \\
$\delta_{\max}$ & 0.5 & 0.5 & 0.5 \\
\bottomrule
\end{tabular}

\label{tab:hyperparams}
\end{table}

\section{Scale of the CART Formulation}
\label{sec:cart_scale_cons}

The number of strategic variables alone understates the computational scale of
the CART benchmark. Each strategic facility selection controls a large
patient-level assignment and scheduling formulation. Across the 100 test
instances with $m=16$, the dimensions of the complete monolithic formulations
are reported in Table~\ref{tab:cart_monolithic_dimensions}.

\begin{table}[H]
\caption{Dimensions of the complete monolithic CART formulations in the
extended $m=16$ test set.}
\centering
\small
\setlength{\tabcolsep}{5pt}
\begin{tabular}{lrrr}
\toprule
\textbf{Model component}
& \textbf{Minimum}
& \textbf{Maximum}
& \textbf{Mean} \\
\midrule
Constraints
& 2,284,802
& 92,139,833
& 46,766,602 \\

Binary variables
& 532,496
& 21,532,176
& 10,928,170 \\

Continuous variables
& 1,177,490
& 47,445,565
& 24,082,020 \\

General integer variables
& 8,336
& 336,456
& 170,768 \\

Total variables
& 1,718,322
& 69,314,197
& 35,180,958 \\
\bottomrule
\end{tabular}
\label{tab:cart_monolithic_dimensions}
\end{table}

EnerGNN selects 3.21 of the 16 candidate manufacturing facilities on average.
Fixing the strategic facility set substantially reduces the induced
patient-level scheduling formulations, as shown in
Table~\ref{tab:cart_reduced_dimensions}.

\begin{table}[H]
\caption{Mean dimensions of the complete and EnerGNN-reduced CART
formulations in the extended $m=16$ setting.}
\centering
\small
\setlength{\tabcolsep}{5pt}
\begin{tabular}{lrrr}
\toprule
\textbf{Model component}
& \makecell{\textbf{Full model mean}}
& \makecell{\textbf{EnerGNN-reduced mean}}
& \makecell{\textbf{Reduction (\%)}} \\
\midrule
Constraints
& 46,766,602
& 10,819,588
& 76.9 \\

Binary variables
& 10,928,170
& 2,434,726
& 77.7 \\

Continuous variables
& 24,082,020
& 5,499,214
& 77.2 \\

General integer variables
& 170,768
& 170,756
& 0.01 \\

Total variables
& 35,180,958
& 8,104,696
& 77.0 \\
\bottomrule
\end{tabular}
\label{tab:cart_reduced_dimensions}
\end{table}

The numbers of constraints, binary variables, continuous variables, and total
variables are reduced by approximately 77\% on average. The number of general
integer variables changes very little because these variables are primarily
associated with patients and scheduling periods rather than with the candidate
manufacturing facilities.

\section{Primal Gap Distributions}
\label{app:distributions}

Figure~\ref{fig:gap_distributions} shows the primal gap distribution 
across test instances for all learning-based methods on each benchmark. 
The x-axis uses a log scale to show both near zero and high gap 
distributions simultaneously. EnerGNN maintains a tight distribution 
near zero across all three benchmarks, confirming that the means 
reported in Table~\ref{tab:main_results} are representative. Primal gaps are computed relative to the global optimum.

\begin{figure}[H]
\centering
\includegraphics[width=\textwidth]{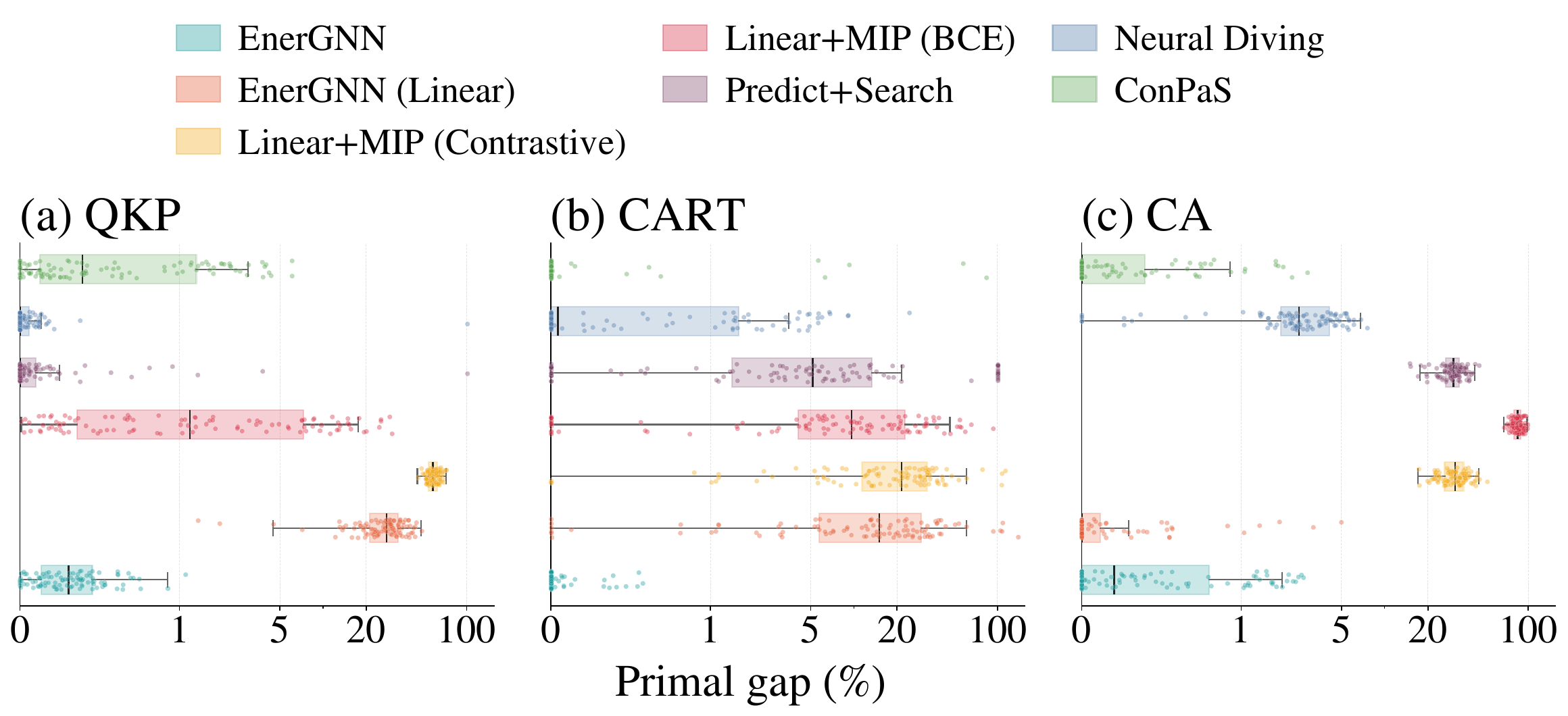}
\caption{Primal gap distributions for all learning-based methods. Infeasible outputs are assigned a 
100\% gap penalty.}
\label{fig:gap_distributions}
\end{figure}

\section{Additional Zero-Shot Constraint Injection Results}
\label{app:zero_shot}

We provide additional zero-shot constraint injection experiments in which the
trained energy is unchanged and only the feasible region of the inference MIP
is modified. In CART, cardinality and facility activation constraints restrict
the available strategic configurations. We use the linear head EnerGNN model for CA, as it achieved the best performance in the main experiments. Optimality gaps are measured relative to reference optima computed under the same injected constraints. These experiments complement the stronger constraint shifts reported in
Section~\ref{sec:experiments}, where variables selected in the original
optimal solution are explicitly made unavailable. 

\begin{table}[H]
\setlength{\belowcaptionskip}{1pt}
\caption{Additional zero-shot constraint injection results on CART.}
\centering
\small
\setlength{\tabcolsep}{5pt}
\begin{tabular}{lcc}
\toprule
\textbf{Injected constraint} & Primal Gap (\%) & $\leq1\%$ (\%) \\
\midrule
$\sum_m y_m \le 6$ & 0.06 & 100 \\
$\sum_m y_m \le 5$ & 0.07 & 100 \\
$\sum_m y_m \le 4$ & 0.12 & 98 \\
$\sum_m y_m \le 3$ & 0.30 & 93 \\
$y_{m_{10}}=0$ & 0.05 & 100 \\
$y_{m_{10}}=0,\; y_{m_{13}}=0$ & 0.03 & 100 \\
$y_{m_{7}}=0,\; y_{m_{10}}=0 ,\; y_{m_{13}}=0$ & 0.09 & 99 \\
$y_{m_{0}}=0,\; y_{m_{3}}=0 ,\; y_{m_{6}}=0$ & 0.07 & 99 \\
$y_{m_{0}}=1$ & 0.09 & 99 \\
$y_{m_{10}}=1$ & 1.75 & 79 \\
$y_{m_{0}}=1,\; y_{m_{10}}=1$ & 0.00 & 100 \\
$y_{m_{10}}=0,\; y_{m_{13}}=1$ & 0.99 & 89 \\
$y_{m_{10}}=0,\; y_{m_{13}}=0,\; y_{m_{1}}=1 ,\; y_{m_{7}}=1$ & 4.83 & 54 \\
$\sum_m y_m \le 4, y_{m_{10}}=0,\; y_{m_{13}}=0,\; y_{m_{1}}=1 ,\; y_{m_{7}}=1$ & 3.55 & 55 \\
$\sum_m y_m \le 5, y_{m_{0}}=1,\; y_{m_{10}}=1$ & 0.00 & 100 \\

\bottomrule
\end{tabular}
\label{tab:zero_shot_cart_appendix}
\end{table}

\begin{table}[H]
\setlength{\belowcaptionskip}{1pt}
\caption{Additional zero-shot constraint injection results on CA.}
\centering
\small
\setlength{\tabcolsep}{5pt}
\begin{tabular}{lcc}
\toprule
\textbf{Injected constraint}
& Primal Gap (\%)
& $\leq1\%$ (\%) \\
\midrule
$\sum_b y_b \le 30$ & 0.20 & 95 \\
$\sum_b y_b \le 20$ & 0.22 & 94 \\
$\sum_b y_b \le 15$ & 0.24 & 92 \\
$\sum_b y_b \le 10$ & 0.17 & 96 \\
$\sum_b y_b \le 5$ & 0.14 & 94 \\
\bottomrule
\end{tabular}
\label{tab:zero_shot_ca_appendix}
\end{table}

\section{Cross-scale Generalization with Nonlinear Energy Head on Combinatorial Auctions}
\label{app:cross_scale_ca}

For CA, the main paper reports cross-scale generalization using the
linear head variant of EnerGNN, which achieves the best performance
on this benchmark. For completeness, Table~\ref{tab:ca_nonlinear_scale}
reports the corresponding results using the nonlinear ReLU energy head,
evaluated without retraining.

\begin{table}[h]
\centering
\small
\caption{CA cross-scale generalization with the nonlinear energy head.}
\label{tab:ca_nonlinear_scale}
\begin{tabular}{lccc}
\toprule
\textbf{Train $\rightarrow$ test}
& \textbf{Primal Gap (\%)}
& \textbf{$\leq$1\% (\%)}
& \textbf{Feasibility (\%)} \\
\midrule
$175 \rightarrow 175$ & 0.53 & 77 & 100 \\
$175 \rightarrow 250$ & 1.40 & 45 & 100 \\
\bottomrule
\end{tabular}
\end{table}

\section{Ablation Study Details}


\subsection{Energy Head Architecture Ablation}
\label{app:head_arch}

We evaluate the effect of energy head width and depth on QKP and CART. In all cases,
the resulting Neural-MIP is solved to global optimality with respect to the
learned energy. Increasing the width or depth of the embedded ReLU head adds
activation binaries to the inference MIP. Across both benchmarks, larger heads
do not systematically improve solution quality, while deeper and wider
architectures can substantially increase solution time.


\begin{table}[h]
\centering
\setlength{\tabcolsep}{3pt}
\caption{Energy head architecture ablation on QKP.}
\label{tab:head_arch_qkp}
\begin{tabular}{lcccccc}
\toprule
\textbf{Head}
& \textbf{Layers}
& \makecell{\textbf{ReLU}\\\textbf{binaries}}
& \makecell{\textbf{Total}\\\textbf{binaries}}
& \makecell{\textbf{Gap}\\\textbf{(\%)}}
& \makecell{\textbf{$\leq1\%$}\\\textbf{(\%)}}
& \makecell{\textbf{Time}\\\textbf{(s)}} \\
\midrule
MLP-16 & 1 & 16  & 166 & 0.47 & 97 & 1.07 \\
MLP-16 & 2 & 32  & 182 & 0.31 & 100 & 2.77 \\
MLP-32 & 1 & 32  & 182 & 0.31 & 99 & 3.56 \\
MLP-32 & 2 & 64  & 214 & 0.35 & 100 & 19.78 \\
MLP-64 & 1 & 64  & 214 & 0.42 & 99 & 4.72 \\
MLP-64 & 2 & 128 & 278 & 0.37 & 99 & 186.23 \\
\bottomrule
\end{tabular}
\end{table}

\begin{table}[h]
\centering
\setlength{\tabcolsep}{3pt}
\caption{Energy head architecture ablation on CART.}
\label{tab:head_arch_cart}
\begin{tabular}{lccccc}
\toprule
\textbf{Head}
& \textbf{Layers}
& \makecell{\textbf{ReLU}\\\textbf{binaries}}
& \makecell{\textbf{Total}\\\textbf{binaries}}
& \makecell{\textbf{Gap}\\\textbf{(\%)}}
& \makecell{\textbf{Time}\\\textbf{(s)}} \\
\midrule
MLP-16 & 1 & 16  & 36  & 0.09 & 0.59 \\
MLP-16 & 2 & 32  & 52  & 0.13 & 1.33 \\
MLP-32 & 1 & 32  & 52  & 0.05 & 1.88 \\
MLP-32 & 2 & 64  & 84  & 0.02 & 3.55 \\
MLP-64 & 1 & 64  & 84  & 0.10 & 3.11 \\
MLP-64 & 2 & 128 & 148 & 0.07 & 24.69 \\
\bottomrule
\end{tabular}
\end{table}

For QKP, increasing head capacity does not consistently improve solution
quality. For example, the one-layer width-16 model achieves a 0.47\% mean
gap in 1.07~s, while the two-layer width-64 model achieves a similar
0.37\% gap but requires 186.23~s. The same pattern appears in CART. Increasing from a one-layer width-16 head to a two-layer width-64 head raises inference time from 0.59 s to 24.69 s, while the mean gap changes only from 0.09\% to 0.07\%.

These results illustrate the expressiveness--tractability trade-off of the
Neural-MIP formulation. Increasing the capacity of the embedded energy head
adds ReLU binaries and can substantially increase solution time without a
corresponding improvement in solution quality. Importantly, the GNN encoder is
evaluated before optimization and is therefore not part of the inference MIP.
Representational capacity can consequently be placed primarily in the graph
encoder while keeping the optimization-compatible energy head compact.

\subsection{Loss Variant Ablation}
\label{app:ablation}
We evaluate three loss configurations on each benchmark:

\begin{itemize}[leftmargin=1.5em]
    \item \textbf{Full loss}: all three components active 
    ($\mathcal{L}_\text{con} + \mathcal{L}_\text{cal} + 
    \mathcal{L}_\text{inf}$).
    \item \textbf{No primal gap term}: $\lambda_\text{cal} = 0$, contrastive and 
    infeasibility only.
    \item \textbf{No infeasibility term}: $\lambda_\text{inf} = 0$, 
    contrastive and calibration only.
\end{itemize}

\subsection{Strict Feasible--Infeasible Dominance}
\label{app:strict_dominance}

The standard training objective does not enforce that every feasible candidate
must have lower energy than every infeasible candidate. To test whether such a
complete ordering would be beneficial, we evaluate a stricter alternative. Let

\[
E_{\max}^{F}
=
\max_{i\in\mathcal F}
E_\theta(G,\mathbf y_i)
\]

denote the highest feasible energy in the current candidate pool. During
backpropagation, $E_{\max}^{F}$ is treated as constant. We replace the original
infeasibility shaping term with

\[
\mathcal L_{\mathrm{inf\text{-}margin}}
=
\frac{1}{|\mathcal I|}
\sum_{k\in\mathcal I}
\left[
E_{\max}^{F}
+
\log(1+\nu_k)
-
E_\theta(G,\mathbf y_k)
\right]_+ ,
\]

where $[a]_+=\max(0,a)$ and $\nu_k$ is the normalized constraint violation.
This term penalizes an infeasible candidate unless its energy lies at least
$\log(1+\nu_k)$ above the highest energy feasible candidate in the sampled
candidate pool, thereby explicitly enforcing feasible--infeasible ordering
within that pool.

We evaluate this alternative under the same Combinatorial Auctions protocol
using EnerGNN with the linear head, which is the best performing EnerGNN variant
for CA.

\begin{table}[H]
\setlength{\belowcaptionskip}{1pt}
\caption{Effect of enforcing strict feasible--infeasible dominance on CA.}
\centering
\small
\setlength{\tabcolsep}{5pt}
\begin{tabular}{lccccc}
\toprule
\textbf{Infeasibility loss}
& \makecell{\textbf{Mean gap}\\\textbf{(\%)}}
& \makecell{\textbf{Exact optimum}\\\textbf{(\%)}}
& \makecell{\textbf{$\leq1\%$}\\\textbf{(\%)}}
& \makecell{\textbf{Feasible}\\\textbf{(\%)}}
& \makecell{\textbf{Mean time}\\\textbf{(s)}} \\
\midrule
Current calibration      & 0.20 & 67 & 95 & 100 & 0.28 \\
Strict dominance margin  & 3.96 & 4  & 18 & 100 & 0.25 \\
\bottomrule
\end{tabular}
\label{tab:strict_dominance}
\end{table}

The stricter loss enforces the desired ordering but substantially reduces
solution quality: the mean gap increases from 0.20\% to 3.96\%, exact recovery
decreases from 67\% to 4\%, and the fraction of solutions within 1\% of
optimality decreases from 95\% to 18\%. Both variants return feasible solutions
on every test instance.

This result clarifies the role of infeasibility supervision. The ranking and
calibration terms are not intended to learn a global feasibility certificate,
feasibility with respect to the encoded constraints is enforced by the inference
MIP. Complete feasible--infeasible separation is therefore not required when
those constraints are explicitly encoded, and in this experiment enforcing it
substantially degrades optimization quality.

\subsection{Readout Design Ablation}
\label{app:readout_ablation}

We evaluate how the choice of MIP-compatible readout affects solution quality
and inference time. The proposed readouts combine pooled learned embeddings
with compact structural summaries. We compare them with simpler embedding-only
readouts and with variants that concatenate the complete binary decision vector
$\mathbf y$ directly to the learned embedding pools. The latter reduces the need
for structured summaries but produces an order-dependent, dimension-specific
input.

\begin{table}[H]
\setlength{\belowcaptionskip}{1pt}
\caption{Readout design ablation on CA and QKP.}
\centering
\small
\setlength{\tabcolsep}{4pt}
\begin{tabular}{llcccc}
\toprule
\textbf{Problem}
& \textbf{Readout}
& \makecell{\textbf{Mean gap}\\\textbf{(\%)}}
& \makecell{\textbf{Exact}\\\textbf{(\%)}}
& \makecell{\textbf{$\leq1\%$}\\\textbf{(\%)}}
& \makecell{\textbf{Time}\\\textbf{(s)}} \\
\midrule
CA  & Embedding pools only             & 0.20 & 68 & 95 & 0.25 \\
CA  & Full proposed readout            & 0.20 & 67 & 95 & 0.28 \\
CA  & Embedding pools + full $\mathbf y$ & 0.21 & 67 & 95 & 0.25 \\
\midrule
QKP & Selected item pool only          & 0.81 & 0  & 72 & 0.50 \\
QKP & Full proposed readout            & 0.31 & 4  & 99 & 3.56 \\
QKP & Selected item pool + full $\mathbf y$ & 0.94 & 0 & 61 & 1.17 \\
\bottomrule
\end{tabular}
\label{tab:readout_ablation}
\end{table}

For CA, removing the compact scalar summaries leaves performance essentially
unchanged, while directly exposing the complete decision vector provides no
improvement. For QKP, the structured readout is substantially more effective:
the mean gap decreases from 0.81\% to 0.31\% and the fraction within 1\%
increases from 72\% to 99\%. Directly concatenating the full decision vector
performs worse.

These results indicate that EnerGNN does not require a large collection of
hand-crafted features. A practical readout should remain compact,
permutation-invariant where possible, compatible with varying problem sizes,
and exactly MIP-representable. Pooled learned embeddings can provide most of
the representation, while compact structural summaries are useful when they
capture meaningful consequences of the selected configuration.

\section{Baseline Implementation Details}
\label{app:baselines}

All neural methods use the same bipartite GNN encoder architecture,
with method-specific prediction or energy heads. We describe these heads,
training objectives, and inference procedures below.

\paragraph{Predict+Search.}
Predict+Search applies a single linear layer to each variable embedding,
followed by a sigmoid, to predict its probability $p_j$ of taking value one.
The model is trained using binary cross entropy supervision on the variable
assignments of positive candidate solutions. At inference, we rank variables
by confidence $|p_j-0.5|$ and fix the most confident 85\% for QKP and 80\%
for CA to their predicted values $\mathbf{1}[p_j\geq0.5]$.
For CART, we instead use a two-sided confidence band, where facility variables
with $p_j\geq0.6$ are fixed to one, those with $p_j\leq0.4$ are fixed to
zero, and the remaining facility variables are left free.
We then solve the reduced MIP.

\paragraph{Linear+MIP.}
Linear+MIP applies a single linear layer to each variable embedding to
produce a scalar score $s_j$. These scores define an additive objective
$s(y)=\sum_j s_j y_j$, which is maximized under the original hard constraints.
We evaluate two training variants with the same architecture:
(i) binary cross entropy supervision on variable assignments, using
$\sigma(s_j)$ as the predicted probability, and
(ii) contrastive supervision over candidate solutions, using the energy
$E(y)=-s(y)$. At inference, both variants optimize the learned linear
objective over all decision variables using a MIP solver.

\paragraph{Neural Diving.}
Our Neural Diving~\citep{nair2020solving} baseline uses a two-layer
variable-level MLP head, with one hidden layer of width $d_h$, a ReLU
activation, and dropout, followed by a sigmoid, to predict an independent
Bernoulli probability $p_j$ for each binary decision variable. The model is trained with binary cross entropy supervision
on the variable assignments of positive candidate solutions. At inference, we
perform $K$ stochastic dives by sampling candidate assignments from these
Bernoulli probabilities. For each dive, a confidence-selected subset of
variables is fixed to the sampled assignment, and the reduced MIP is
optimized over the remaining free variables.
The final solution is the best feasible solution found across all dives. We use $K=20$ dives for QKP, $K=8$ for CART, and $K=100$ for CA.
The smaller dive budget for CART reflects the higher computational cost
of solving its reduced MIPs. For QKP and CA, we fix
the most confident 85\% and 50\% of variables, respectively. For CART, because
the number of facility variables is small, we instead use the same two-sided
confidence band as Predict+Search: variables with $p_j\geq0.6$ or $p_j\leq0.4$
are fixed to their sampled values, and the remaining variables are left free.

\paragraph{ConPaS.}
Our controlled ConPaS-style~\citep{huang2024contrastive} baseline applies
a 3-layer variable-level MLP head, with 2 hidden layers of width 64
and ReLU activations, followed by a sigmoid, to predict independent
Bernoulli probabilities $p_j$ for the binary decision variables.
It scores a candidate assignment by its factorized log probability,
\[
s(y)=\sum_j y_j \log p_j + (1-y_j)\log(1-p_j).
\]
Training uses an InfoNCE contrastive loss over the candidate pool, encouraging
positive solutions to receive higher scores than suboptimal feasible and
infeasible candidates. At inference, the neural model is not embedded in the
MIP. Instead, variables are ranked by $p_j$. Let $X_0$ denote the $k_0$
lowest probability variables and $X_1$ the $k_1$ highest probability variables.
We solve the original MIP with the additional rank-neighbourhood constraint
\[
\sum_{j\in X_0} y_j + \sum_{j\in X_1} (1-y_j) \leq \Delta,
\]
which allows at most $\Delta$ disagreements with the predicted low/high
probability sets while leaving all other variables unrestricted. Increasing $\Delta$ relaxes this rank-neighbourhood restriction by allowing
more disagreements with the neural ranking. This can improve solution quality
under constraint shift, where variables predicted to be unlikely may become
necessary because of exogenous constraints, but it also enlarges the MIP search
space. We therefore tune $\Delta$ per benchmark to balance solution quality and
solve time.
The rank-neighbourhood parameters are $(k_0,k_1,\Delta)=(55,20,5)$ for QKP,
$(8,2,1)$ for CART, and $(88,0,3)$ for CA. For CA, we keep the paper-scaled
choice $k_0=\lfloor n/2 \rfloor$ used by ConPaS. The original in-distribution ConPaS setting uses $\Delta=0$. However,
under our constraint shift setting, exogenous outage and force-on constraints
can make a small number of low-ranked bids necessary. We therefore introduce a
small slack of $\Delta=3$. We also evaluated the stricter paper-style setting
$(88,0,0)$, but it produced poor solution quality, with a mean gap of $6.23\%$,
despite being fast, with an average solve time of $0.05$s compared with
Gurobi's $0.14$s.


\section*{Broader Impact}

This work develops learning-augmented optimization methods for constrained
combinatorial decision-making. A potential positive impact is to make learned
decision systems more reliable in settings where feasibility is essential, such
as logistics, supply chain planning, manufacturing, and healthcare operations.
Rather than producing decisions directly, EnerGNN learns an objective landscape
that is optimized under explicit hard constraints. This may help decision-support
tools adapt to changing demand, capacity, or operational requirements while
preserving feasibility through the optimization layer.
In personalized medicine supply chains, such methods could support the evaluation
of strategic planning decisions under different patient demand and capacity
scenarios. More broadly, the approach may be useful in domains where learned
models must interact with exact algorithms to produce valid decisions.
The method is not intended to replace domain expertise, clinical judgement, or
human oversight in high-stakes settings. Exact constrained inference preserves
feasibility with respect to the specified constraints, but it does not guarantee
that the learned energy ranks all decisions correctly or that the specified
objective captures all clinical, ethical, or operational priorities. If deployed
without careful validation, a learned energy could favor suboptimal allocations
or reinforce undesirable trade-offs.

\end{document}